\documentclass[10pt,twocolumn]{article}
  \newdimen\paravsp  \paravsp=1.3ex % named paragraph spacing

\usepackage[
    a4paper,
    paperheight=11in,
    textheight=9.25in,
    textwidth=6.75in,
    columnsep=0.25in,
    nohead, % foot kept for page numbers/notice
    footnotesep=9pt plus 4pt minus 2pt,
    centering
]{geometry}

\usepackage[utf8]{inputenc}
\usepackage[T1]{fontenc}
\usepackage{xspace}  % used for macros, see shortcuts.tex
\usepackage{amsfonts}  % add additional mathematical fonts, e.g. \mathbb
\usepackage{mathtools} % amsmath with fixes and additions
\usepackage{hyperref} % options: [hypertex,breaklinks,draft]
\usepackage{graphicx} % used for including graphics
\usepackage{subcaption}  % for adding sub-figures
\usepackage{algorithm}  % for pseudo-code environment
\usepackage{algpseudocode} % for pseudo-code environment
\usepackage{natbib}
\newtheorem{theorem}{Theorem}

\newtheorem{lemma}[theorem]{Lemma}
\newtheorem{definition}[theorem]{Definition}

\newtheorem{proposition}[theorem]{Proposition}

\newenvironment{keywords}{\centerline{\bf\small
    Keywords}\begin{quote}\small}{\par\end{quote}\vskip 1ex}
\newenvironment{proof}{{\noindent\bf Proof.}}{\vskip 1ex}

\def\paradot#1{\vspace{\paravsp plus 0.5\paravsp minus 0.5\paravsp}\noindent{\bf\boldmath{#1.}}} % boldface paragraph.
\newcommand{\sref}{Section~\ref}
\newcommand{\appref}[1]{Appendix~\ref{#1}}
\newcommand{\aref}[1]{Algorithm~\ref{#1}}

\newcommand{\fref}[1]{Figure~\ref{#1}}

\newcommand{\tref}[1]{Table~\ref{#1}}

\newcommand{\realSet}{\ensuremath{\mathbb{R}}\xspace}

\newcommand{\expectation}{\ensuremath{\mathbb{E}}\xspace}

\DeclareMathOperator*{\argmax}{argmax}
\newcommand{\prob}{\ensuremath{Pr}\xspace}

\newcounter{comment}
\newcommand{\agentSet}{\ensuremath{\mathcal{I}}\xspace}

\newcommand{\stSpace}{\ensuremath{\mathcal{S}}\xspace}

\newcommand{\st}{\ensuremath{s}\xspace}

\newcommand{\stp}{\ensuremath{s'}\xspace}

\newcommand{\actSpace}{\ensuremath{\mathcal{A}}\xspace}
\newcommand{\actSet}[1]{\ensuremath{\actSpace_{#1}}\xspace}
\newcommand{\actSpaceJoint}{\ensuremath{\vec{\actSpace}}\xspace}
\newcommand{\act}{\ensuremath{a}\xspace}

\newcommand{\actSub}[1]{\ensuremath{\act_{#1}}\xspace}
\newcommand{\actSubt}[2][t]{\ensuremath{\act_{#2,#1}}\xspace}
\newcommand{\actJoint}{\ensuremath{\vec{\act}}\xspace}

\newcommand{\obsSpace}{\ensuremath{\mathcal{O}}\xspace}
\newcommand{\obsSet}[1]{\ensuremath{\obsSpace_{#1}}\xspace}

\newcommand{\obs}{\ensuremath{o}\xspace}

\newcommand{\obsSub}[1]{\ensuremath{\obs_{#1}}\xspace}
\newcommand{\obsSubt}[2][t]{\ensuremath{\obs_{#2,#1}}\xspace}
\newcommand{\obsJoint}{\ensuremath{\vec{\obs}}\xspace}

\newcommand{\transF}{\ensuremath{\mathcal{T}}\xspace}

\newcommand{\obsF}{\ensuremath{\mathcal{Z}}\xspace}
\newcommand{\obsFSub}[1]{\ensuremath{\obsF_{#1}}\xspace}

\newcommand{\rewF}{\ensuremath{\mathcal{R}}\xspace}

\newcommand{\rewFSub}[1]{\ensuremath{\rewF_{#1}}\xspace}
\newcommand{\rew}{\ensuremath{r}\xspace}

\newcommand{\rewSub}[1]{\ensuremath{\rew_{#1}}\xspace}
\newcommand{\rewSubt}[2][t]{\ensuremath{\rew_{#2,#1}}\xspace}
\newcommand{\rewVec}{\ensuremath{\vec{\rew}}\xspace}

\newcommand{\ret}{\ensuremath{G}\xspace}

\newcommand{\retSubt}[2][t]{\ensuremath{\ret_{#2,#1}}\xspace}

\newcommand{\disc}{\ensuremath{\gamma}\xspace}

\newcommand{\hist}{\ensuremath{h}\xspace}

\newcommand{\histJoint}{\ensuremath{\vec{\hist}}\xspace}

\newcommand{\histSpace}{\ensuremath{\mathcal{H}}\xspace}

\newcommand{\histSub}[1]{\ensuremath{\hist_{#1}}\xspace}
\newcommand{\histSubt}[2][t]{\ensuremath{\histSub{#2,#1}}\xspace}
\newcommand{\histSet}[1]{\ensuremath{\histSpace_{#1}}\xspace}
\newcommand{\histSett}[2][t]{\ensuremath{\histSpace_{#2,#1}}\xspace}

\newcommand{\bel}{\ensuremath{b}\xspace}
\newcommand{\belt}[1][t]{\ensuremath{\bel_{#1}}\xspace}

\newcommand{\belSub}[1]{\ensuremath{\bel_{#1}}\xspace}

\newcommand{\belInit}{\ensuremath{\belt[0]}\xspace}

\newcommand{\posgTuple}{\ensuremath{\langle \agentSet, \stSpace, \belInit, \actSpaceJoint, \{ \obsSet{i} \}, \transF, \{ \obsFSub{i} \}, \{ \rewFSub{i} \} \rangle}\xspace}

\newcommand{\histSt}{\ensuremath{w}\xspace}
\newcommand{\histStSpace}{\ensuremath{\mathcal{W}}\xspace}

\newcommand{\pol}{\ensuremath{\pi}\xspace}

\newcommand{\polSub}[1]{\ensuremath{\pi_{#1}}\xspace}

\newcommand{\gen}{\ensuremath{\mathcal{G}}\xspace}

\newcommand{\algname}{Partially Observable Type-based Meta Monte-Carlo Planning\xspace}
\newcommand{\algacronym}{POTMMCP\xspace}
\newcommand{\hpsname}{history-policy-state\xspace}
\newcommand{\hpsnames}{history-policy-states\xspace}

\newcommand{\resultfigheight}{0.22\textheight}
\newcommand{\envfigheight}{0.175\textheight}

\begin{document}
%%%%%%%%%%%%%%%%%%%%%%%%%%%%%%%%%%%%%%%%%%%%%%%%%%%%%%%%%%%%%%%
%%                    T i t l e - P a g e                    %%
%%%%%%%%%%%%%%%%%%%%%%%%%%%%%%%%%%%%%%%%%%%%%%%%%%%%%%%%%%%%%%%

\title{
    \vspace{-4ex}
    \vskip 2mm\bf\Large\hrule height5pt \vskip 4mm
    Combining a Meta-Policy and Monte-Carlo Planning for Scalable Type-Based Reasoning in Partially Observable Environments
    \vskip 4mm \hrule height2pt
}
\author{
    Jonathon~Schwartz$^{1}$\thanks{Corresponding author: Jonathon.schwartz@anu.edu.au}, 
    Hanna~Kurniawati$^{1}$, 
    Marcus~Hutter$^{1,2}$\\[3mm]
    \begin{minipage}{0.33\textwidth}
        \centering  
        \normalsize $^{1}$Australian National University \\
        \normalsize Canberra Australia
    \end{minipage}
    \begin{minipage}{0.33\textwidth}
        \centering
        \normalsize $^{2}$Google DeepMind \\
        \normalsize London, United Kingdom
    \end{minipage}
}

\date{} % Final date added as timestamp by arXiv
\maketitle

\begin{abstract}
    The design of autonomous agents that can interact effectively with other agents without prior coordination is a core problem in multi-agent systems. Type-based reasoning methods achieve this by maintaining a belief over a set of potential behaviours for the other agents. However, current methods are limited in that they assume full observability of the state and actions of the other agent or do not scale efficiently to larger problems with longer planning horizons. Addressing these limitations, we propose \algname (\algacronym) ---an online Monte-Carlo Tree Search based planning method for type-based reasoning in large partially observable environments. \algacronym incorporates a novel meta-policy for guiding search and evaluating beliefs, allowing it to search more effectively to longer horizons using less planning time. We show that our method converges to the optimal solution in the limit and empirically demonstrate that it effectively adapts online to diverse sets of other agents across a range of environments. Comparisons with the state-of-the art method on problems with up to $10^{14}$ states and $10^8$ observations indicate that \algacronym is able to compute better solutions significantly faster\footnote{This is an updated and expanded version of the paper that appeared in the proceeding of the AAMAS 2023 conference \cite{schwartzBayesAdaptiveMonteCarloPlanning2023}}.
\end{abstract}

\vspace{2ex}
\begin{keywords}
    \vspace{-2ex}
    \small
    Multi-Agent, POSG, Type-Based Reasoning, Planning under Uncertainty, MCTS
\end{keywords}

\section{Introduction}\label{sec:Intro}
A core research area in multi-agent systems is the development of autonomous agents that can interact effectively with other agents without prior coordination \citep{bowlingCoordinationAdaptationImpromptu2005,stoneAdHocAutonomous2010,albrechtSpecialIssueMultiagent2017}. Type-based reasoning methods give agents this ability by maintaining a belief over a set \textit{types} for the other agents \citep{barrettEmpiricalEvaluationAd2011,albrechtConvergenceOptimalityBestresponse2014,barrettCooperatingUnknownTeammates2015,albrechtBeliefTruthHypothesised2016}. Each type is a mapping from the agent’s interaction history to a probability distribution over actions, and specifies the agent’s behaviour. If the set of types is sufficiently representative, type-based reasoning methods can lead to fast adaptation and effective interaction without prior coordination  \citep{albrechtGametheoreticModelBestresponse2013,barrettCooperatingUnknownTeammates2015}.

Unfortunately, type-based reasoning significantly increases the size and complexity of the planning problem and finding scalable and efficient solution methods remains a key challenge. This is especially true in partially observable settings where the planning agent is unable to observe the type of the other agent, their interaction history, or the state of the environment. In this setting the agent must maintain a joint belief over these three features leading to a belief space that grows exponentially with the planning horizon and number of agents. Several online planning methods based on Monte-Carlo Tree Search (MCTS) have shown promising performance in non-trivial partially observable problems \citep{kakarlapudiDecisionTheoreticPlanningCommunication2022, schwartz2022intmcp}. However, so far these methods have only been demonstrated in settings where the other agent's type is known and scale poorly to domains with longer planning horizons.

Inspired by the success of techniques combining MCTS with a search policy in single-agent \citep{schrittwieserMasteringAtariGo2020} and zero-sum \citep{silverGeneralReinforcementLearning2018} settings, in this paper we propose a method for integrating a search policy into planning in the type-based, partially-observable setting. The use of a search policy offers a number of advantages. Firstly, it guides exploration, biasing it away from low value actions and allowing the agent to plan effectively for longer horizons. Secondly, if the search policy has a value function this can be used for evaluation during search. Doing this avoids expensive Monte-Carlo (MC) rollouts and can significantly improve search efficiency.  

To alleviate the disadvantages that come with using a search policy, we propose a novel meta-policy using the set of types available to the planning agent. The meta-policy is generated using an empirical game \citep{wellmanMethodsEmpiricalGametheoretic2006} which computes the expected payoffs between each pairing of types using a number of simulated episodes. This makes the meta-policy relatively inexpensive to compute and side-steps the usual method for finding a search policy which is to train one from scratch \citep{silverGeneralReinforcementLearning2018,brownCombiningDeepReinforcement2020,timbersApproximateExploitabilityLearning2022a,liCombiningTreeSearchGenerative2023}.

Combining the meta-policy with MCTS, we create a new online planning algorithm for type-based reasoning in partially observable environments, which we refer to as \algname (\algacronym). Through extensive evaluations and ablations on large competitive, cooperative, and mixed partially observable environments - the largest of which has four agents and on the order of $10^{14}$ states and $10^{8}$ observations - we demonstrate empirically that \algacronym is able to substantially outperform the existing state-of-the-art method \citep{kakarlapudiDecisionTheoreticPlanningCommunication2022} in terms of final performance and planning time. Additionally, we prove the correctness of our approach, showing that \algacronym converges to the Bayes-optimal policy in the limit.

\section{Related Work}\label{sec:RelatedWork}
\paradot{Monte-Carlo Planning} We are interested in MC planning methods for environments where the agent must adapt to a set of possible types of other agents. When coordination between agents is involved, this is the \textit{ad-hoc teamwork} problem \citep{bowlingCoordinationAdaptationImpromptu2005,stoneAdHocAutonomous2010}. Various approaches to ad-hoc teamwork have been proposed, including those based on stage games \citep{wuOnlinePlanningAd2011}, Bayesian beliefs \citep{barrettEmpiricalEvaluationAd2011}, the Partially Observable Markov Decision Process (POMDP) \citep{barrettCommunicatingUnknownTeammates2014}, types with parameters \citep{albrechtReasoningHypotheticalAgent2017}, and for the many agent setting \citep{yourdshahiLargeScaleAdhoc2018}. All these methods use MCTS but are limited to environments where the state and actions of the other agents are fully observed. In the \textit{agent modelling} setting \citep{albrechtAutonomousAgentsModelling2018}, several MCTS-based methods have been proposed for the Interactive POMDP (I-POMDP) \citep{gmytrasiewiczFrameworkSequentialPlanning2005} framework. Including methods based on finite state-automata \citep{panellaInteractivePOMDPsFinitestate2017} and nested MCTS \citep{schwartz2022intmcp}, as well as methods for open multi-agent systems \citep{eckScalableDecisiontheoreticPlanning2020}, and systems with communication \citep{kakarlapudiDecisionTheoreticPlanningCommunication2022}. Other works have focused on planning in strictly cooperative \citep{czechowskiDecentralizedMCTSLearned2021,choudhuryScalableOnlinePlanning2022} or competitive \citep{cowlingInformationSetMonte2012a} settings. Also related to our work are a number of Bayes-adaptive planning methods using MCTS \citep{guezScalableEfficientBayesadaptive2013,amatoScalablePlanningLearning2014,kattLearningPOMDPsMonte2017}. However, these methods focus on learning parameters of the environment's transition dynamics, while we focus instead on learning the policy type and history of the other agent.

\paradot{Combining Reinforcement Learning and Search} A number of methods have been proposed that combine Reinforcement Learning (RL) with MCTS. Self-play RL and MCTS have been combined in two-player fully observable zero-sum games with a known environment model \citep{silverMasteringGameGo2016,silverGeneralReinforcementLearning2018} and using a learned model \citep{schrittwieserMasteringAtariGo2020}. Similar methods have been applied to zero-sum imperfect-information games \citep{brownSuperhumanAIMultiplayer2019,brownCombiningDeepReinforcement2020}, as well as cooperative games where there is prior coordination for decentralized execution \citep{lererImprovingPoliciesSearch2020}. 
Our method builds on this line of research, specifically relating to using an existing policy as a prior for search. However, we apply these advances outside of self-play zero-sum games or where there is prior coordination, instead focusing on online adaption to previously unknown other agents. There have also been works looking at combining MCTS with PUCT and RL for training a best-response policy to a single known policy \citep{timbersApproximateExploitabilityLearning2022a} or a distribution over policies \citep{liCombiningTreeSearchGenerative2023}. Compared with these methods, our method does not rely on any training to generate the search policy used by PUCT from scratch for a given policy set. Instead we propose an efficient method for utilizing the information available to the planning agent in the type-based reasoning setting to improve the planning without any training needed, even if the set of policies changes.

\section{Problem Description}\label{sec:ProblemDescription}
We consider the problem of \textit{type-based} reasoning in partially observable environments. We model the problem as a Partially Observable Stochastic Game (POSG) \citep{hansenDynamicProgrammingPartially2004} which consists of $N$ agents indexed $\mathcal{I} = \{1, \dots, N$\}, a discrete set of states $\stSpace$, an initial state distribution $\belInit \in \Delta(\stSpace)$, the joint action space $\actSpaceJoint = \actSet{1} \times \dots \times \actSet{N}$, the finite set of observations $\obsSet{i}$ for each agent $i \in \mathcal{I}$, a state transition function $\transF: \stSpace \times \actSpaceJoint \times \stSpace \rightarrow [0, 1]$ specifying the probability of transitioning to state $\stp$ given joint action $\actJoint$ was performed in state $\st$, an observation function for each agent $\obsFSub{i}: \stSpace \times \actSpaceJoint \times \obsSet{i} \rightarrow [0, 1]$ specifying the probability that performing joint action $\actJoint$ in state $\st$ results in observation $\obsSub{i}$ for agent $i$, and a bounded reward function for each agent $\rewFSub{i}: \stSpace \times \actSpaceJoint \rightarrow \realSet$. For convenience, we also define the generative model $\gen$, which combines $\transF, \obsF, \rewF$, and returns the next state, joint observation, and joint reward, given the current state and joint action $\langle \stp, \obsJoint, \rewVec \rangle \sim \gen(\st, \actJoint)$.

At each step, each agent $i \in \mathcal{I}$ simultaneously performs an action $\actSub{i} \in \actSet{i}$ from the current state $\st$ and receives an observation $\obsSub{i} \in \obsSet{i}$ and reward $\rewSub{i} \in \realSet$. Each agent has no direct access to the environment state or knowledge of the other agent's actions and observations. Instead they must rely only on information in their \textit{interaction-history} up to the current time step $t$: $\histSubt{i} = \langle \obsSubt[0]{i} \actSubt[0]{i} \obsSubt[1]{i} \actSubt[1]{i} \dots \actSubt[t-1]{i} \obsSubt[t]{i} \rangle$\footnote{For clarity the time subscript $t$ is omitted where it is clear from context.}. The set of all time $t$ histories for agent $i$ is denoted $\histSett{i}$. Agents select their next action using their \textit{policy} $\polSub{i}$ which is a mapping from their history $\histSubt{i}$ to a probability distribution over their actions, where $\polSub{i}(\actSub{i} | \histSubt{i})$ denotes the probability of agent $i$ performing action $\actSub{i}$ given history $\histSubt{i}$.  

We assume the other agents are using policies from a known fixed set of policies where each policy corresponds to an agent type. We denote the planning agent by $i$, and all other agents collectively using $-i$. The set of fixed policies for the other agents is $\Pi_{-i} = \{ \pi_{-i, m} | m = 1, \dots, M \}$, where $M$ is the number of policies in the set and $\pi_{-i} = \{ \pi_{j} | j \in \mathcal{I}\setminus\{i\}\}$ is a joint policy that assigns a policy for each non-planning agent. We denote the set of policies available for a specific agent as $\Pi_{j}$ for $j \in \mathcal{I}$. Furthermore, we assume the joint-policy used by the other agents is selected based on a known prior distribution $\rho$, where $\rho(\pi_{-i, m}) = \prob(\pi_{-i, m})$ \footnote{$\rho$ assigns prior probability to each joint-policy in $\Pi_{-i}$ which are not permutation-invariant in general. If $N > 2$ and agents are symmetric then there may be multiple equivalent joint-policies, each with a prior probability in $\rho$.}.  

The goal of the planning agent is to maximize its expected \textit{return} $\retSubt{i}$ with respect to $\rho$ within a single episode, $\retSubt{i} = \expectation_{\pi_{-i, m} \sim \rho} \left[ \sum_{k=t}^{\infty} \disc^{k-t} \rewSubt[k]{i} | \pi_{-i, m} \right]$, where $\disc \in [0, 1)$ is the discount. The Bayes-optimal policy is the policy that achieves the maximum possible expected return. 

\begin{figure}[t]
    \centering
    \includegraphics[width=\linewidth, height=0.15\textheight, keepaspectratio]{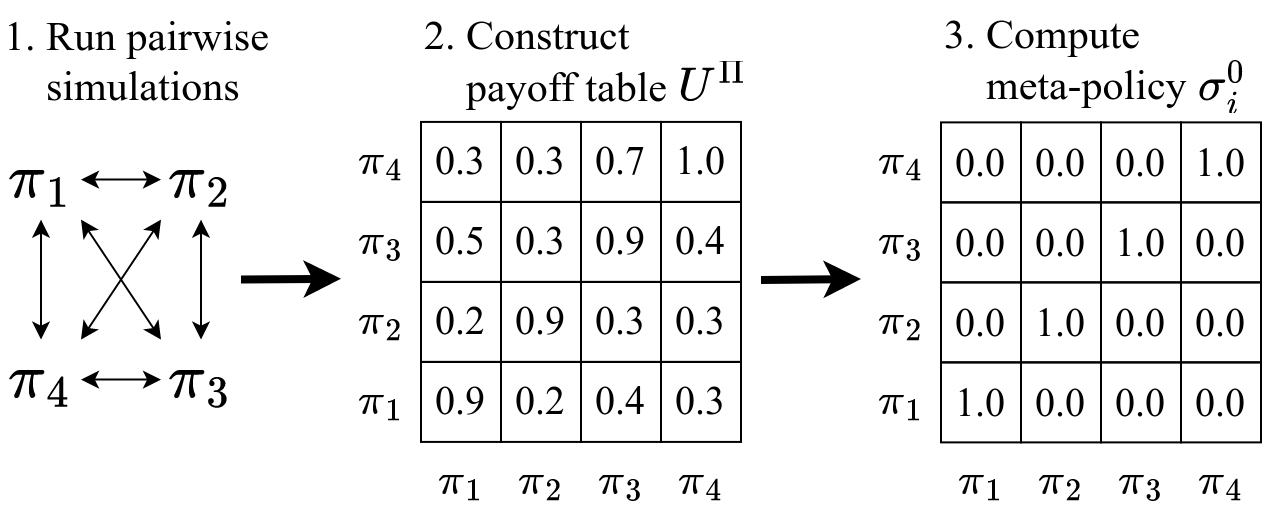}
    \caption{The greedy meta-policy $\sigma_{i}^{0}$ generation process for a symmetric environment with two agents.}
    \label{fig:meta_pi}
\end{figure}

\section{Method}\label{sec:Method}
Here we present \algacronym, an online MCTS-based planning algorithm for type-based reasoning in partially observable environments. Like existing planners  \citep{eckScalableDecisiontheoreticPlanning2020, kakarlapudiDecisionTheoreticPlanningCommunication2022, schwartz2022intmcp}, \algacronym uses MCTS to calculate the planning agent's best action from its current belief $\bel_{i}$. However, it offers several important improvements over existing algorithms. Firstly, it incorporates the PUCT algorithm \citep{ silverGeneralReinforcementLearning2018} for selecting actions during search. PUCT can significantly improve planning efficiency by biasing search towards the most relevant actions according a \textit{search policy}. This makes it possible to plan for longer horizons, as well as offers improved integration of value functions for leaf node evaluation. To address the limitation of PUCT, namely that it relies on access to a good search-policy, the second improvement offered by \algacronym is the use of a novel meta-policy as the \textit{search-policy}. The meta-policy has the advantage that it can be efficiently generated from the policy set $\Pi$, and offers a robust prior since it considers performance across the entire set of other agent policies.

\subsection{Meta-Policy} \label{sssec:meta-policy}

In this work, a meta-policy $\sigma_{i}$ is a function mapping a joint policy to a mixture over individual policies. For our purposes it is a mapping from the set of other agent joint policies to a distribution over the set of valid policies for the planning agent $\sigma_{i}: \Pi_{-i} \rightarrow \Delta(\Pi_{i})$, so that $\sigma_{i}(\polSub{i, k} | \polSub{-i, m}) = \prob(\polSub{i, k} | \polSub{-i, m})$ for $\polSub{i, k} \in \Pi_{i}, \polSub{-i, m} \in \Pi_{-i}$. In symmetric environments, the set of valid policies for the planning agent $i$ is the set of all individual policies for any of the other agents $\Pi_{-i} = \bigcup_{j \neq i} \Pi_{j}$. In asymmetric environments, it may be necessary to have access to a distinct set of policies for agent $i$. In practice, these could be any policies used to train the other agent policies in $\Pi_{-i}$ or could be a separate set of policies generated from data or heuristics. 

Ideally the meta-policy would map from the planning agent's belief to a mixture over policies. However, if we think of each policy as an action we can see that finding a mapping from beliefs to a mixture over policies has the same challenges as finding a mapping from beliefs to the primitive actions of the underlying POSG. Instead we propose a meta-policy $\sigma_{i}$ as described - mapping from the set of other agent joint policies to a distribution over the set of valid policies for the planning agent - that is efficient to compute and which can then be used to improve planning over primitive actions. Furthermore, this meta-policy has the added advantage that it is relatively inexpensive to adapt if $\Pi$ changes, for example if a new type is added between episodes.

The main idea is to generate the meta-policy using an empirical game constructed from the policies in $\Pi_{-i} \cup \Pi_{i}$. An \textbf{empirical game}, much smaller in size than the full game, is a normal-form game where the actions are policies and the expected returns for each joint policy are estimated from sample games \citep{walshAnalyzingComplexStrategic2002, wellmanMethodsEmpiricalGametheoretic2006, lanctotUnifiedGametheoreticApproach2017}. Formally, an empirical game is a tuple $\langle \Pi, U^{\Pi}, N \rangle$ where $N$ is the number of players, $\Pi = \langle \Pi_{1}, \dots, \Pi_{N} \rangle = \Pi_{-i} \cup \Pi_{i}$ is the set of policies for all players and $U^{\Pi} : \Pi \rightarrow \realSet^{N}$ is a payoff table of expected returns (averaged over multiple games) for each joint policy played by all players, with $U^{\Pi}_i(\pi_i,\pi_{-i})\in\mathbb{R}$ denoting the payoff for player $i$ when using policy $\pi_{i}$ against the other agents using joint policy $\pi_{-i}$ (an example is shown in \fref{fig:meta_pi}). Where a joint policy is an assignment of a policy $\polSub{j} \in \Pi_{j}$ for each agent $j \in 1, \dots, N$. If agent's are symmetric then we take the average over permutations of the same joint policy (joint policies with same individual policies but assigned to different agents).  

Empirical games have the advantage that they can be very efficient to compute, since they require only a finite number of simulations and each simulation can be very fast to run depending on the representation of the policies and the environment. For example, each simulation takes less than a second for our environments with policies represented as neural networks. For slow-to-query policies, such as those that use search, the process will be slower but the process of computing the payoffs is trivially parallelizable and accuracy can be traded-off with time by changing the number of simulations used. It may also be possible to approximate slow-to-query policies with a fast policy by training a function approximator using imitation learning as suggested by \cite{timbersApproximateExploitabilityLearning2022a}. Furthermore, when a new policy is added to the set $\Pi$, only the payoffs for that policy need to be computed, which requires only $|\Pi|$ pay-offs to be computed, making it relatively inexpensive to add a new policy.

The next question is, given the empirical-game payoffs $U^{\Pi}$ how should we define the meta-policy. Since the meta-policy will ultimately be used to guide search for the planning agent, we would like it to select the policy from the set $\Pi_{i}$ that maximizes performance of the planning agent. One possible method is to select the policy $\polSub{i} \in \Pi_{i}$ that has the highest payoff against a given other agent policy $\polSub{-i}$. However, it is possible that the best response policy from the set $\Pi_{i}$ may change during an episode depending on the planning agent's current belief, while $U^{\Pi}$ is defined based on the expected performance from the start of an episode. This means that a meta-policy that only selects the maximizing policy from the set $\Pi_{i}$ with respect to the payoff table $U^{\Pi}$ has the potential to select a sub-optimal policy with respect to the planning agent's current belief. 

To protect against the potential pitfall's of a myopic meta-policy, we propose a flexible meta-policy based on the softmax function. Specifically, we define the softmax meta-policy $\sigma_{i}^{\tau}$ where, $\sigma_{i}^{\tau}(\polSub{i} | \polSub{-i}) = \frac{1}{\eta}\exp[\frac{1}{\tau}U^{\Pi}(\polSub{i}, \polSub{-i})]$ with normalizing constant $\eta = \sum_{\polSub{i}} \exp[\frac{1}{\tau}U^{\Pi}(\polSub{i}, \polSub{-i})]$ and temperature hyperparameter $\tau$ which controls how uniform or greedy the policy is. As $\tau \to 0$ the meta-policy becomes greedier, with $\sigma_{i}^{0}$ being the greedy policy that selects $\polSub{i} \in \Pi_{i}$ that maximizes the empirical payoff $U^{\Pi}$. Conversely, as $\tau \to \infty$ becomes more uniform, with $\sigma_{i}^{\infty}(\polSub{i} | \polSub{-i}) = 1/|\Pi_i|$ being the uniform meta-policy. 

The process of computing a meta-policy is relatively straightforward and needs only be done once for a given set of policies $\Pi$. \fref{fig:meta_pi} shows a high-level overview of the process for computing a greedy meta-policy $\sigma_{i}^{0}$. Firstly, each joint policy in the set $\Pi$ is simulated, for some number of episodes (we use 1000 in our experiments). Next, the average payoff for each pair of policies is used to construct the empirical game payoff table $U^{\Pi}$. Finally, the meta-policy is computed according to its equation using the payoff table $U^{\Pi}$. If a policy is added or removed from the set $\Pi$ (between episodes) only the entries for this policy need to be added/removed, before the meta-policy can be computed again. This makes it fairly easy to use with different variations of policy sets, which can be useful for rapid experimentation or for making quick adjustments depending on the settings where it is being used. 

\subsection{\algacronym}

We now present \algacronym, which consists of agent beliefs over \hpsnames, and MCTS using PUCT and the meta-policy for selecting actions from each belief. Importantly, \algacronym does not utilize knowledge of the actual policy the other agent is using during planning. Rather \algacronym maintains a joint belief over the environment state, the possible policies of the other agent, as well as the other agent's history. The meta-policy then computes the search-policy from each belief, conditioned on that belief's distribution over the other agent's policy.

\subsubsection{Beliefs over \hpsnames}

To correctly model the environment and other agents, \algacronym maintains a belief over the other agents' histories $\histSub{-i} \in \histSet{-i}$, their policies $\pi_{-i} \in \Pi_{-i}$, and the environment state $s \in \stSpace$. Each belief is thus a distribution over history-policy-state tuples, which we refer to as \hpsnames and denote using $\histSt$ and its components using dot notation: $\histSt.\st$, $\histSt.\histJoint$, $\histSt.\pi_{-i}$. The space of \hpsnames is denoted $\histStSpace$. Using this notation, the planning agent's belief is a distribution over \hpsnames, where $\bel_{i}(\histSt | \hist_{i}) = \prob(\histSt_{t} = \histSt | \hist_{i, t} = \hist_{i})$.

Defining the belief using \hpsnames transforms the original POSG problem into a POMDP for the planning agent where the other agents' histories and policies, and the environment state are learned online. This conversion is analogous to the one employed by the more general I-POMDP framework \citep{gmytrasiewiczFrameworkSequentialPlanning2005}. However, unlike the I-POMDP, we only consider the possible policies of the other agents, and use their history to represent their internal state, rather than an explicit belief. 

\subsubsection{Meta-policy MCTS with \hpsnames}

\algacronym extends the POMCP \citep{silverMonteCarloPlanningLarge2010} algorithm to planning with beliefs over \hpsnames. Being an online planner, each step \algacronym executes a search to find the next action, followed by particle filtering to update the agent's belief given the most recent observation. To do this \algacronym builds a search tree $T$ of agent histories using the PUCT algorithm \citep{rosinMultiarmedBanditsEpisode2011, silverGeneralReinforcementLearning2018} and a meta-policy as the search policy. Each node of the tree corresponds to a history, where $T(\hist_{i})$ denotes the node for history $\hist_{i}$, and maintains an approximate belief over \hpsnames $\hat{\bel}_{i}(\hist_{i})$ represented using a set of unweighted particles where each particle corresponds to a sample \hpsname $\histSt$. For each action $\actSub{i} \in \actSet{i}$ from $\hist_{i}$ there is an edge $\hist_{i}\actSub{i}$ that stores a set of statistics $\langle N(\hist_{i}\actSub{i}), P(\actSub{i} | \hist_{i}), W(\hist_{i}\actSub{i}), Q(\hist_{i}\actSub{i}) \rangle$, where $N(\hist_{i}\actSub{i})$ is the visit count, $P(\actSub{i}|\hist_{i})$ is the prior probability of selecting $\actSub{i}$ given $\hist_{i}$, $W(\hist_{i}\actSub{i})$ is the total action-value, and $Q(\hist_{i}\actSub{i})$ is the mean action-value.

For each real step at time $t$ in the environment, \algacronym
constructs the tree $T$ rooted at the agent's current history \histSub{i, t} via a series of simulated episodes. Each simulation starts from a \hpsname sampled from the root belief $\hat{\belSub{i}}(\histSub{i, t})$ and proceeds in three stages. Pseudo-code for the search procedure is shown in \aref{alg:search}.

\begin{algorithm}[t]
    \caption{\algacronym Search}
    \label{alg:search}
    \begin{algorithmic}[0]
        \Procedure{Search}{$\histSub{i}$}
            \While{search time limit not reached}
                \State $\histSt \sim \hat{\bel_{i}}(\cdot, \histSub{i})$
                \State $\polSub{i} \sim \sigma_{i}(\histSt.\polSub{-i})$
                \State \Call{Simulate}{$\histSt$, $\histSub{i}$, $\polSub{i}$, $0$} 
            \EndWhile
            \State \textbf{return} $\argmax_{\actSub{i} \in \actSet{i}} N(\histSub{i}\actSub{i})$
        \EndProcedure
    \end{algorithmic}
        
    \begin{algorithmic}[0]
        \Procedure{Simulate}{$\histSt$, $\histSub{i}$, $\polSub{i}$, $depth$}
            \If {$\disc^{depth} < \epsilon$}
                \State \textbf{return} 0
            \EndIf

            \If{$\histSub{i} \notin T$}
                \State \textbf{return} \Call{Expand}{$\histSub{i}$, $\polSub{i}$}
            \EndIf
            
            \State $\actSub{i} \leftarrow$ \Call{PUCT}{$\histSub{i}$}
            \State $\actSub{-i} \sim \histSt.\pol_{-i}(\cdot | \histSt.\hist_{-i})$
            \State $\langle s', \obsJoint, \rewVec \rangle \sim \gen(\histSt.\st, \langle \actSub{i}, \actSub{-i} \rangle)$
            
            \State $\histSt' \leftarrow \langle \st', \histSt.\histSub{-i}\actSub{-i}\obsSub{-i}, \histSt.\pi_{-i} \rangle$
            \State $G_{i} \leftarrow \rewSub{i} + \disc$ \Call{Simulate}{$\histSt', \histSub{i}\actSub{i}\obsSub{i}, \polSub{i}, depth+1$}
            \State $\hat{\bel}_{i}(\histSub{i}\actSub{i}\obsSub{i}) \leftarrow \hat{\bel}_{i}(\histSub{i}\actSub{i}\obsSub{i}) \cup \{\histSt'\}$
            \State $N(\histSub{i}\actSub{i}) \leftarrow N(\histSub{i}\actSub{i}) + 1$
            \State $W(\histSub{i}\actSub{i}) \leftarrow W(\histSub{i}\actSub{i}) + G_{i}$
            \State $Q(\histSub{i}\actSub{i}) \leftarrow \frac{W(\histSub{i}\actSub{i})}{N(\histSub{i}\actSub{i})}$
            \For {$\hat{\actSub{i}} \in \actSet{i}$}
                \State $P(\hat{\actSub{i}} | \histSub{i}) \leftarrow P(\hat{\actSub{i}} | \histSub{i}) +  \frac{\polSub{i}(\hat{\actSub{i}}|\histSub{i}) - P(\hat{\actSub{i}} | \histSub{i})}{N(\histSub{i})}$
            \EndFor

            \State \textbf{return} $G_{i}$
        \EndProcedure
    \end{algorithmic}
\end{algorithm}

\begin{algorithm}[t]
    \caption{\algacronym PUCT Action selection}
    \label{alg:puct}
    \begin{algorithmic}[0]
        \Procedure{PUCT}{$\histSub{i}$}
            \For {$\actSub{i} \in \actSet{i}$}
                \State $U(\histSub{i}\actSub{i}) \leftarrow c (P(\actSub{i}|\histSub{i}) (1 - \lambda) + \frac{\lambda}{|\actSet{i}|}) \frac{\sqrt{N(\hist_{i})}}{1 + N(\hist_{i}\actSub{i})}$
            \EndFor
            \State \textbf{return} $\underset{\actSub{i} \in \actSet{i}}{\argmax} \left\{ Q(\hist_{i}\actSub{i}) + U(\histSub{i}\actSub{i}) \right\}$
        \EndProcedure
    \end{algorithmic}
\end{algorithm}

\begin{algorithm}[t]
    \caption{\algacronym Value Function Node Expansion}
    \label{alg:expand}
    \begin{algorithmic}[0]
        \Procedure{EXPAND}{$\histSub{i}$, $\polSub{i}$}
            \For {$\actSub{i} \in \actSet{i}$}
                \State $N(\histSub{i}\actSub{i}) \leftarrow 0$
                \State $P(\actSub{i}|\histSub{i}) \leftarrow \polSub{i}(\actSub{i}|\histSub{i})$
                \State $W(\histSub{i}\actSub{i}) \leftarrow 0$
                \State $Q(\histSub{i}\actSub{i}) \leftarrow 0$
            \EndFor
            \State \textbf{return} $V^{\polSub{i}}(\histSub{i})$
        \EndProcedure
    \end{algorithmic}
\end{algorithm}

In the first stage, until a leaf node is reached, actions for the planning agent $i$ are selected using the PUCT algorithm (\aref{alg:puct}), while actions for the other agent $-i$ are sampled using their policy and history contained within the sampled \hpsname: $\actSub{-i} \sim \histSt.\pi_{-i}(\histSt.\hist_{-i})$
Our version of PUCT action selection adds uniform exploration noise $1/|\actSet{i}|$ with a mix-in proportion $\lambda$, while the constant $c$ controls the influence of the exploration value $U(\histSub{i}\actSub{i}$) relative to the action-value $Q(\histSub{i}\actSub{i})$. This differs from \cite{silverGeneralReinforcementLearning2018} where they sample the noise from a Dirichlet distribution for the current root node only and where the noise is used for exploration during training over multiple episodes. In this work we are instead interested in balancing exploration within a single episode and so utilize uniform noise as motivated by Theorem~\ref{thm:convergence}. Depending on the values of the constants $\lambda$ and $c$, each action will eventually be explored even if the search-policy assigns it zero probability, and ensures the search can always find the optimal action given enough planning time. 

In the second stage, upon reaching a leaf node ($\histSub{i} \notin T$), the leaf node is evaluated and expanded by adding it to the tree and adding an edge for each action (\aref{alg:expand}). Evaluation of the leaf node involves estimating two properties of the node; the value $v_{i}$ and the policy $p_{i}$. Both of these are computed using a policy from the set $\Pi_{i}$ which is sampled according to the meta-policy and the other agent's policy contained in the \hpsname particle $\polSub{i} \sim \sigma_{i}(\cdot |\histSt.\polSub{-i})$ \footnote{Importantly, note that the other agent policy $\histSt.\polSub{-i}$ is sampled according to the planning agents belief $\hat{\bel_{i}}$ and so may be different to the true policy of the other agent.}. Estimating $v_{i}$ assumes that the policy $\polSub{i}$ has a value function, which is the case for policies generated using most learning and planning methods. However, if a value function is not available $v_{i}$ can be estimated using a MC-rollout instead. Each edge $\histSub{i}\actSub{i}$ from the leaf node is initialized to $\langle N(\hist_{i}\actSub{i}) = 0, P(\actSub{i}|\hist_{i}) = p_{i}(\actSub{i}), W(\hist_{i}\actSub{i}) = 0, Q(\hist_{i}\actSub{i}) = 0 \rangle$. The value estimate of the leaf node $v_{i}$ is not used to initialize the edges but instead used in the last stage of the simulation. 

In the third and final stage, the statistics for each edge along the simulated trajectory are updated by propagating the value $v_{i}$ from the leaf node back-up to the root node of the tree. The policy prior for each edge $\histSub{i}\actSub{i}$ along this path are also updated by averaging over the existing prior and the latest policy $P(\actSub{i} | \histSub{i}) \leftarrow P(\actSub{i} | \histSub{i}) + [\polSub{i}(\actSub{i}|\histSub{i}) - P(\actSub{i} | \histSub{i})]/N(\hist_{i})$. This is a crucial difference between ours and previous methods. Previous methods apply PUCT to trees where each node is treated as a fully-observed state and so only compute the prior once when the node is first expanded \citep{silverGeneralReinforcementLearning2018, schrittwieserMasteringAtariGo2020}. In our setting each node in the tree is an estimate of the planning agent's belief. When a node is first expanded it contains only a single particle and so is likely inaccurate, and thus the policy prior will also be inaccurate. As the node is visited during subsequent simulations the belief accuracy improves and so in our method we iteratively update the policy prior to reflect this. Thus as the number of visits to a belief increases, i.e. $N(\histSub{i}) \rightarrow \infty$, we have:

\begin{equation*}
    P(\actSub{i} | \histSub{i}) \rightarrow \sum_{\polSub{-i} \in \Pi_{-i}} \bel_{i}(\polSub{-i} | \histSub{i}) \sum_{\polSub{i} \in \Pi_{i}} \sigma_{i}(\polSub{i} | \polSub{-i}) \polSub{i}(\actSub{i} | \histSub{i}) 
\end{equation*}

Where $\bel_{i}(\polSub{-i} | \histSub{i})$ is the true posterior belief over the other agent's policy $\polSub{-i}$ given the planning agent's history. In this way $P(\actSub{i} | \histSub{i})$ is a function of the belief over the other agent's policy. 

Once search is complete, the planning agent selects the action at the root node with the greatest visit count $\actSub{i, t} = \argmax_{\actSub{i}} N(\histSub{i, t}\actSub{i})$ and receives an observation $\obsSub{i, t+1}$ from the real environment. At this point $T(\histSub{i, t}\actSub{i, t}\obsSub{i, t+1})$ is set as the new root node of the search tree and $\hat{\belSub{i}}(\histSub{i, t}\actSub{i, t}\obsSub{i, t+1})$ the new root belief.

\section{Theoretical Properties}
In this section we show that \algacronym converges to the Bayes-optimal policy with respect to the policy set $\Pi_{-i}$ and prior $\rho$. The proof is based on the conversion of the problem to a POMDP, which allows us to apply the analysis in \citet{silverMonteCarloPlanningLarge2010}. We point out however, that the original analysis was based on using the UCB algorithm \citep{auerFinitetimeAnalysisMultiarmed2002}. We extend their proof to apply to the PUCB algorithm \citep{rosinMultiarmedBanditsEpisode2011}, which requires an additional assumption on the prior probabilities assigned to each action in order to ensure sufficient exploration during search for convergence in the limit to be guaranteed. Note, in our implementation we use the $\lambda$ parameter which can be chosen so that the assumption is met even if the the search-policy prior assigns zero probability to some actions.

Define $V(\histSub{i}) = \underset{\actSub{i} \in \actSet{i}}{\max} Q(\histSub{i}\actSub{i}) \,\, \forall \histSub{i} \in \histSet{i}$.

\begin{theorem} \label{thm:convergence}
    For all $\epsilon > 0$ (the numerical precision, see \aref{alg:search}), given a suitably chosen c (e.g. $c > \frac{R_{max}}{1-\gamma}$) and prior probabilities $P(\actSub{i}|\histSub{i}) > 0, \forall \histSub{i} \in \histSpace_{i}, \actSub{i} \in \actSet{i}$ (e.g. $\lambda > 0$), from history $\histSub{i}$ \algacronym constructs a value function at the root node that converges in probability to an $\epsilon'$-optimal value function, $V(\histSub{i}) \xrightarrow{p} V^{*}_{\epsilon'}(\histSub{i})$, where $\epsilon' = \frac{\epsilon}{1 - \gamma}$. As the number of visits $N(\histSub{i})$ approaches infinity, the bias of $V(\histSub{i})$ is $O(\log{N(\histSub{i})}/N(\histSub{i})).$
\end{theorem}

\begin{proof} (sketch, full proof is provided in the \appref{sup:proofs})
    A POSG with a set of stationary policies for the other agent, and a prior over this set is a POMDP, so the analysis from \citet{silverMonteCarloPlanningLarge2010} applies to \algacronym, noting that for $N(\histSub{i})$ sufficiently large PUCB has the same regret bounds as UCB given each action is given prior probability $P(\histSub{i}\actSub{i}) > 0, \forall \histSub{i} \in \histSpace_{i}, \actSub{i} \in \actSet{i}$ (\citet{rosinMultiarmedBanditsEpisode2011}, Thm. 2 and Cor. 2) and so the same analysis of POMCP using UCT applies to POMCP using PUCT. 
\end{proof}

\begin{figure}[t]
    \centering
    \includegraphics[width=\linewidth, height=0.175\textheight, keepaspectratio]{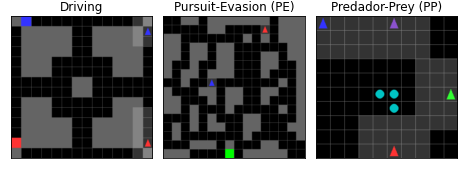}
    \caption{Experiment Environments}
    \label{fig:envs}
\end{figure}

\section{Experiments}
We ran experiments on three benchmark environments designed to evaluate \algacronym against existing state-of-the-art methods across a diverse set of scenarios. In addition, we conducted a number of ablations to investigate the different meta-policies and the learning dynamics of our method. \footnote{All code used for the experiments is available at \url{https://github.com/Jjschwartz/potmmcp}}

\subsection{Multi-Agent Environments}

For our experiments we used one cooperative, one competitive, and one mixed environment (\fref{fig:envs}). The largest of which has four agents and on the order of $10^{14}$ states and $10^{8}$ observations. Each environment models a practical real-world scenario; namely, navigation, pursuit-evasion, and ad-hoc teamwork. These environments add additional complexities to existing benchmarks \cite{eckScalableDecisiontheoreticPlanning2020,kakarlapudiDecisionTheoreticPlanningCommunication2022} and were chosen in order to assess \algacronym's ability across a range of domains that required both planning over many steps and reasoning about the other agent's behaviour. See \appref{sup:envs} for further details on each environment.

\textbf{Driving:} A general-sum grid world navigation problem requiring coordination \citep{lererLearningExistingSocial2019}. Each agent is tasked with driving a vehicle from start to destination while avoiding crashing into other vehicles.

\textbf{Pursuit-Evasion (PE):} A two-agent asymmetric zero-sum grid world problem where the \textit{evader} has to reach a safe location without being observed by the \textit{pursuer} \citep{seamanNestedReasoningAutonomous2018, schwartz2022intmcp}.

\textbf{Predator-Prey (PP):} A co-operative grid world problem where multiple predator agents must work together to catch prey \citep{loweMultiagentActorcriticMixed2017}. The \textit{two-agent} version has two predators with each prey requiring two predators to capture. The \textit{four-agent} version has four predators with each prey requiring three predators to capture.

\subsection{Policies}

For each environment, a diverse set of four to five policies was created and used for the set $\Pi$. These policies were used for the other agent during evaluations and also for the meta-policy $\sigma_{i}$ and policy prior $\rho$. In our experiments each policy was a deep neural network trained using the PPO RL algorithm \citep{schulmanProximalPolicyOptimization2017} and different multi-agent training schemes. See \appref{sup:fixed_policies} for details including multi-agent training scheme, training hyper-parameters, and empirical-game payoff matrices. 

\subsection{Baselines} \label{ssec:baselines}

We compared \algacronym against a number of baselines in each environment. For the planning baseline, we adapted the current state-of-the-art method for planning in partially observable, typed multi-agent systems \citet{kakarlapudiDecisionTheoreticPlanningCommunication2022} designed for I-POMDPs with communication. The full algorithm incorporates an explicit model of communication and while it can in principle support beliefs over multiple types, it was only tested in the setting where the other agent's type is fixed and known. We adapted their method by removing the explicit communication model and extending it to handle beliefs over multiple-agent types, we refer to this baseline as \textit{I-POMCP-PF}. We also use the same belief update and reinvigoration strategies for both I-POMCP-PF and \algacronym to keep the comparison fair and since both would benefit equally from changes to the belief updates.

\textbf{Meta-policy:} Selects a policy from the set $\Pi$ at the start of the episode based on the meta-policy with respect to the distribution $\rho$. This acts as a lower bound on the performance of \algacronym, and represents an agent with access to the meta-policy but without using beliefs or search. 

\textbf{Best-Response:} This is the best performing policy from the set of fixed policies $\Pi$ against each policy, assuming full-knowledge of the policy of the other agent. This acts as an approximate upper-bound given there is a policy in the set $\Pi$ that is a best-response policy to at least one policy in the set, which is the case in our experiments. 

\textbf{I-POMCP-PF + Random:} The I-POMCP-PF algorithm using the uniform-random policy and MC-rollouts for leaf node evaluations. 

\textbf{I-POMCP-PF + Meta:} The I-POMCP-PF algorithm using the value function of the meta-policy for leaf node evaluations in place of MC-rollouts.

\begin{figure*}[t]
    \centering
    \includegraphics[width=\linewidth, height=\textheight, keepaspectratio]{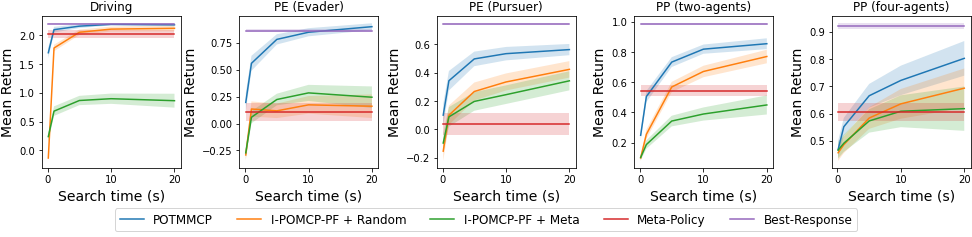}
    \caption{Mean episode return of \algacronym and baseline methods across each environment. Results are for \algacronym and baselines using the softmax $\sigma_{i}^{0.25}$ meta-policy. Shaded areas show the $95\%$ CI.}
    \label{fig:vs_baselines}
\end{figure*}

\begin{figure}[t]
    \centering
    \begin{subfigure}{.5\linewidth}
        \centering
        \includegraphics[width=\linewidth, height=\resultfigheight, keepaspectratio]{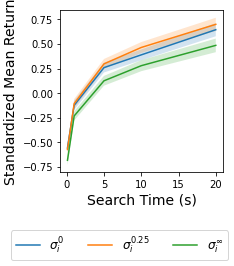}
    \end{subfigure}%
    \begin{subfigure}{.5\linewidth}
        \centering
        \includegraphics[width=\linewidth, height=\resultfigheight, keepaspectratio]{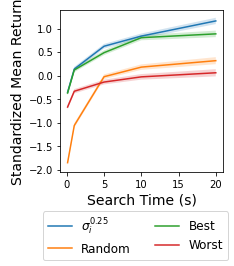}
    \end{subfigure}
    \caption{Evaluation of \algacronym using different search policies. Both plots show standardized mean episode returns averaged across all environments. (Left) Shows comparison of greedy $\sigma_{i}^{0}$, softmax $\sigma_{i}^{0.25}$, and uniform $\sigma_{i}^{\infty}$ meta-policies. (Right) Shows comparison between $\sigma_{i}^{0.25}$, the best and worst performing policy from the set of policies $\Pi_{i}$, and the uniform random policy. Shaded areas show $95 \%$ CI.}
    \label{fig:meta_pi_comparison}
\end{figure}

\subsection{Experimental Setup}

For each experiment we ran a minimum of 400 episodes, or 48 hours of total computation time, whichever came first. We tested each planning method using a search time per step of $N_{time} \in [0.1, 1, 5, 10, 20]$ seconds and $N_{particles} = 100 N_{time}$ particles. As per prior work \citep{silverMonteCarloPlanningLarge2010}, after each real step in the environment rejection sampling was used to add additional particles to the root belief until it contained at least $(1 + 1/16)N_{particles}$. For all experiments we used $\epsilon = 0.01$ and a discount of $\disc = 0.99$. For \algacronym we chose PUCT constants based on prior work \citep{schrittwieserMasteringAtariGo2020}. We used exploration constant $c = 1.25$ along with normalized $Q$-values to handle the returns being outside of $[0, 1]$ bounds in the tested environments. We used $\lambda = 0.5$ for the mix-in proportion, although additional experiments we found that \algacronym was very robust to the value of $\lambda$ in the environments and settings used for our experiments (see \appref{sup:lambda}). For the I-POMCP-PF baselines we used a UCB exploration constant of $c = \sqrt{2}$ \citep{auerFinitetimeAnalysisMultiarmed2002} along with normalized $Q$-values.

\subsection{Evaluation of Returns}

\fref{fig:vs_baselines} shows the mean episode return of \algacronym and the baseline methods in each environment. Given the same planning time, \algacronym outperformed both versions of I-POMCP-PF across all environments. Furthermore, \algacronym even matched the performance of the Best-Response baseline in two environments given enough planning time. We attribute the gains in performance of \algacronym to its ability to effectively utilize the meta-policy for biasing search and for leaf-node evaluation. 

Biasing the action-selection meant less planning time was spent exploring lower value actions (according to the meta-policy) and lead to significantly deeper searches ($\sim 13$ for \algacronym vs $\sim 5$ for I-POMCP-PF for 20 s planning time, see \appref{sup:search_depth}). This translated to a longer effective planning horizon for \algacronym and improved performance. This was especially evident in the PE (Evader) problem which requires long horizon planning as the evader agent must choose between many possible paths to its goal, and the choice it makes early in the episode affects its chances of reaching its goal without being spotted by the pursuer. 

Effectively utilizing the meta-policy's value function for leaf node evaluation meant that \algacronym was able to avoid expensive MC-rollouts and ultimately led to faster simulations and more efficient planning. This helps explain the gains in performance over I-POMCP-PF + Random, however we also observe similar or greater gains over I-POMCP-PF + Meta which, like \algacronym, uses the meta-policy's value function and avoids MC-rollouts. Indeed I-POMCP-PF's performance actually suffers from using the meta-policy when compared to using the random policy with MC-rollouts. This result highlights a limitation of UCB when combined with value-functions. In our experiments where the rewards are not especially dense, we found the difference in value estimates produced by the search-policy between two actions from the same belief can be very small (i.e. a factor of $\gamma$). When using UCB this small difference can be dominated by the variance in returns generated during MC simulations, and translates into UCB being unable to effectively distinguish the best action during search when using the meta-policy's value estimates. PUCT reduces the influence of the variance by biasing the actions according to the search-policy - which by definition selects actions that maximize the value estimates (even if the difference is small between actions). This bias acts to amplify the value of the best actions relative to the other actions and allows \algacronym to more effectively use the meta-policy's value estimates for planning, leading to the gains in planning efficiency and performance. Of course this has the drawback that if the search-policy is bad, then it takes more planning time to overcome the bias introduced and find the optimal actions. The fact that we see improved performance when using the meta-policy across all environments provides empirical evidence that it makes for a good search-policy.

\begin{figure}[t]
    \centering
    \includegraphics[width=\linewidth, height=\resultfigheight, keepaspectratio]{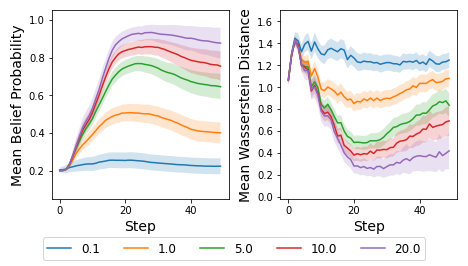}
    \caption{\algacronym's belief accuracy during an episode in the PP (two-agents) environment. (Left) Mean probability assigned by the belief to the other agent's true policy. (Right) Wasserstein distance between the belief's estimated action distribution and the other agent's true action distribution. Each line shows a different planning time. Shaded areas show $95\%$ CI.}
    \label{fig:belief_acc}
\end{figure}

\begin{figure*}[t]
    \centering
    \includegraphics[width=\linewidth, height=\resultfigheight, keepaspectratio]{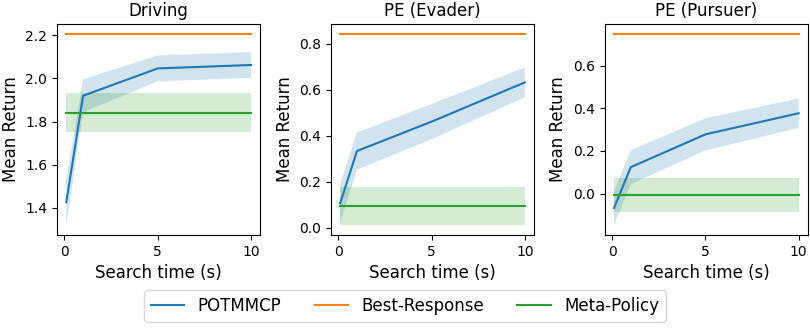}
    \caption{Mean episode return of \algacronym and baseline methods in the Driving and Pursuit-Evasion environments when using a larger policy set $|\Pi_{-i}| = 15$. Results are for \algacronym and baselines using the softmax $\sigma_{i}^{0.25}$ meta-policy. Shaded areas show the $95\%$ CI.}
    \label{fig:many_pi}
\end{figure*}

\subsection{Evaluation of Meta-Policies}

To assess the effect of meta-policy choice on performance we compared \algacronym using greedy $\sigma_{i}^{0}$, softmax $\sigma_{i}^{0.25}$, and uniform $\sigma_{i}^{\infty}$ meta-policies in each environment. \fref{fig:meta_pi_comparison} (left) shows the standardized mean episode returns of \algacronym using the different meta-policies averaged across all environments. Results for each individual environment are available in the \appref{sup:meta_pi}. Overall we found $\sigma_{i}^{0.25}$ was the most robust, performing best or close to best across all environments, and had the highest mean performance when averaged across all environments (although not significantly so). We expect this is likely due to the reasons covered in \sref{sssec:meta-policy}. It's worth noting that we didn't tune the $\tau$ parameter for our experiments and expect marginal performance gains may be seen with tuned values. A question for future work is whether an optimal value for $\tau$ can be inferred from the empirical payoff-matrix. 

To study the benefit of using a meta-policy and robustness to search policy choice, we tested \algacronym using different search policies in each environment. We did this by replacing the meta-policy with each of the policies in the set $\Pi_{i}$, as well as the uniform random policy. \fref{fig:meta_pi_comparison} (right) shows the standardized results averaged across all environments, while the results for each individual environment are available in \appref{sup:meta_vs_fixed}. We found that the meta-policy lead to the most consistent results, having similar or better performance than the best single-policy in each environment. While the best single-policy was able to perform comparable to the meta-policy, we typically observed a significant gap between the best and worst performing policies and no obvious way to tell beforehand which of the policies out the set would perform best/worse. Overall we found the meta-policy lead to the most robust performance without having to test each individual search policy.

\subsection{Evaluation of Beliefs} \label{sec:exp_beliefs}

To better understand the adaptive capabilities of \algacronym we analysed the evolution of beliefs during an episode. \fref{fig:belief_acc} shows the accuracy of \algacronym's belief for the PP (two-agent) problem while the results for all environments are available in \appref{sup:belief_accuracy}. We observed that in all environments \algacronym's belief in the correct other agent type increased as the episode progressed. A similar trend occurred for the action distribution accuracy, with the distance between the estimated and true distributions decreasing over time. We also found roughly monotonic improvement in belief accuracy with increased planning time. It's worth noting we also observed a drop in accuracy for steps that occurred beyond the typical episode length (e.g. $\sim 37$ for the PP (two-agent) problem), this is likely due to increased uncertainty from the environment moving into less common states which can impact the planning agent's beliefs and the behaviour of the other agents. However, even taking this into account, these results empirically demonstrate \algacronym's ability to learn the other agent's type from online interactions and helps explain \algacronym robust performance against a diverse set of policies. 

\subsection{Scalability with Policy Set Size}

So far our we have shown results for experiments where the size of the set of possible other agent policies has been relatively small, $|\Pi_{-i}| \in [4, 5]$. To test how well \algacronym performs compared to the baselines with larger policy sets, we ran additional experiments for the Driving and Pursuit-Evasion environments with $|\Pi_{-i}| \in 15$, the results are shown in \fref{fig:many_pi}. \algacronym's performance scaled well with the larger performance, performing significantly better than the meta-baseline across all experiments and also improving with planning time. Furthermore, due to \algacronym using Monte-Carlo simulation, the computation time per simulation does not increase with policy set size. The key potential limitation in practice is the belief accuracy, since the number of particles needed to represent the space of policies accurately will grow with the policy set size. However, depending on the diversity of the behaviours represented by the different policies, there can be considerable overlap between policy behaviours. In this case the number of particles needed to represent the belief so that it can still be used to find a good policy for the planning agent may not grow that much as the policy set of the other agent grows.

\section{Conclusion}
We presented a scalable planning method for type-based reasoning in large partially observable environments. Our algorithm, \algacronym, offers two key contributions over existing planners. The first is the use of PUCT for action selection during search. The second is a new meta-policy which is used to guide the search. Through extensive evaluations we demonstrate that \algacronym significantly improves on the performance and planning efficiency of existing state-of-the-art methods. Furthermore, we prove that \algacronym will converge in the limit to the Bayes-optimal solution. 

Multiple avenues for future research exist. Extending \algacronym to handle continuous state, action, and observation spaces would make it more widely applicable. We proposed a general and flexible meta-policy and validated it empirically, however based on prior work \citep{lanctotUnifiedGametheoreticApproach2017}, a more principled approach likely exists. Lastly, exploring methods for improving generalization to policies outside of the known set set would be a valuable addition.

%-------------------------------%
\paradot{Acknowledgements}
%-------------------------------%
We would like to thank Kevin Li for his feedback on an earlier draft.
This work is supported by an AGRTP Scholarship and the ANU Futures Scheme.

%%%%%%%%%%%%%%%%%%%%%%%%%%%%%%%%%%%%%%%%%%%%%%%%%%%%%%%%%%%%%%%
%%                 B i b l i o g r a p h y                   %%
%%%%%%%%%%%%%%%%%%%%%%%%%%%%%%%%%%%%%%%%%%%%%%%%%%%%%%%%%%%%%%%

\begin{small}
    \bibliography{ref}
\end{small}

%%%%%%%%%%%%%%%%%%%%%%%%%%%%%%%%%%%%%%%%%%%%%%%%%%%%%%%%%%%%%%%
%%                     A p p e n d i x                       %%
%%%%%%%%%%%%%%%%%%%%%%%%%%%%%%%%%%%%%%%%%%%%%%%%%%%%%%%%%%%%%%%

\newpage
\onecolumn
\appendix

\section{Proofs} \label{sup:proofs}

In this section, we show that \algacronym converges to the Bayes-optimal policy. Our proof first shows how a POSG can be converted into an equivalent POMDP, given the other agent is using a policy from a known distribution. This first step follows a similar construction to I-POMDPs \citep{gmytrasiewiczFrameworkSequentialPlanning2005} for converting a POSG to a POMDP, except we model the other agent's only by their policy and history as opposed to the more general I-POMDP formulation which also includes agent frames and beliefs. Next we show that our algorithm converges to the optimal solution in the derived POMDP (and thus the original POSG with the known other agent policy set) following similar steps to those in \citet{silverMonteCarloPlanningLarge2010}. Noting that we extend their proof to the multi-agent setting, and also to MCTS using PUCT.

\subsection{POSG-POMDP Value Equivalence}

Given the set of policies $\Pi_{-i} = \{\polSub{-i, m} | m = 1, \cdots, M\}$ for the other agent $-i$ and a prior distribution over them $\rho$, where $\rho(\polSub{-i,m}) = Pr(\polSub{-i, m})$, a POSG can be framed as a POMDP for the planning agent $i$. This frames the original POSG as a Bayesian Reinforcement Learning problem where the joint policy for the other agent $\polSub{-i}$ and their history $\histSub{-i}$ become hidden variables within the state space that must be inferred by the planning agent. 

Let $\histSt_{t} = \langle \st, \polSub{-i}, \histSub{-i, t} \rangle$ denote a \hpsname at time $t$. $\histStSpace_{t}$ denotes the space of time $t$ \hpsnames, and $\histStSpace = \{\histStSpace_{t} | 0 \leq t \leq T \}$ denotes the space of all possible \hpsnames for time horizon $T$.  $\histStSpace$ is finite given the action and observation spaces of all agents are finite. For problems with an infinite horizon (i.e. no step limit), discounting is required and the horizon can be set such that the value functions are $\epsilon$-optimal, as suggested by \citep{kocsisBanditBasedMontecarlo2006}. We use dot notation to denote the components of a $\histSt_{t}$ - $\histSt_{t}.\st$, $\histSt_{t}.\polSub{-i}$, $\histSt_{t}.\histSub{-i, t}$. For convenience we use $\histSt_{t}.\act_{-i}$ and $\histSt_{t}.\obs_{-i}$ to denote the last action and observation of agent $-i$ contained in the history $\histSt_{t}.\histSub{-i, t}$, i.e. $\act_{-i, t-1}$ and $\obs_{-i, t}$.

\begin{lemma}
    Given a set of stationary history-based policies $\Pi_{-i}$ for the other agents $-i$, a prior distribution over this set $\rho$, and a POSG $\mathcal{M} = \posgTuple$, consider the derived POMDP $\bar{\mathcal{M}} = \langle \histStSpace, \bar{\belInit}, \actSpace_{i}, \obsSpace_{i}, \bar{\transF}, \bar{\obsF}, \bar{\rewF} \rangle$ for planning agent $i$ with \hpsnames $\histSt_{t} = \langle \st, \polSub{-i}, \histSub{-i, t} \rangle$ as states, where, 
    
    \begin{equation*}
    \begin{aligned}
        \histStSpace &:= \stSpace \times \Pi_{-i} \times \histSet{-i} \\
        \bar{\belInit}(\histSt_{t}) &:= \belInit(\histSt_{t}.\st) \rho(\histSt_{t}.\polSub{-i})\delta(\histSt_{t}.\histSub{-i}, \emptyset) \\
        \bar{\transF}(\histSt_{t}, \act_{i}, \histSt_{t'}) &:= 
            \histSt_{t}.\polSub{-i}(\histSt_{t'}.\act_{-i} | \histSt_{t}.\hist_{-i}) \transF(\histSt_{t}.\st, \langle \act_{i}, \histSt_{t'}.\act_{-i} \rangle, \histSt_{t'}.\st) \\
            & \quad \times \obsFSub{-i}(\histSt_{t'}.\st, \langle \act_{i}, \histSt_{t'}.\act_{-i} \rangle, \histSt_{t'}.\obsSub{-i}) \delta(\langle t', \histSt_{t'}.\polSub{-i} \rangle, \langle t+1, \histSt_{t}.\polSub{-i} \rangle) \\
        \bar{\obsF}(\histSt_{t}, \actSub{i}, \obsSub{i}) &:=
            \obsFSub{i}(\histSt_{t}.\st, \langle \act_{i}, \histSt_{t}.\act_{-i} \rangle, \obs_{i}) \\
        \bar{\rewF}(\histSt_{t}, \act_{i}) &:= 
            \sum_{\actSub{-i} \in \actSet{-i}} \histSt_{t}.\polSub{-i}(\actSub{-i} | \histSt_{t}.\histSub{-i}) \rewFSub{i}(\histSt_{t}.\st, \langle \act_{i}, \act_{-i} \rangle)
    \end{aligned}
    \end{equation*}

    and $\delta$ is the Kronecker Delta function which is $1$ if the function arguments are equal, otherwise $0$, then the value function $\bar{V}^{\polSub{i}}(\hist_{i, t})$ of the derived POMDP is equivalent to the value function $V^{\polSub{i}}(\hist_{i,t})$ of the POSG, $\forall \polSub{i} \bar{V}^{\polSub{i}}(\hist_{i,t}) = V^{\polSub{i}}(\hist_{i, t})$.
    
    \label{lemma:posg_bapomdp}
\end{lemma}

\begin{proof} 
    By backward induction on the Bellman equation, starting from the horizon,
        
    \begin{align*}
        V^{\polSub{i}}(\hist_{i, t}) 
            % POSG bellman equation
            &= \sum_{\st \in \stSpace} 
                \sum_{\polSub{-i} \in \Pi_{-i}} 
                \sum_{\hist_{-i, t} \in \histSpace_{i, t}} 
                \belSub{i}(\langle \st, \polSub{-i}, \histSub{-i, t} \rangle | \hist_{i, t}) 
                \sum_{\actSub{i} \in \actSet{i}} \polSub{i}(\actSub{i} | \hist_{i, t}) 
                \sum_{\actSub{-i} \in \actSet{-i}} \polSub{-i}(\actSub{-i} | \hist_{-i, t}) \\
            & \quad \left[
                \rewFSub{i}(\st, \langle \actSub{i}, \actSub{-i} \rangle)
                + \disc 
                \sum_{\st' \in \stSpace}  \transF(\st, \langle \actSub{i}, \actSub{-i} \rangle, \st') 
                \sum_{\obsSub{i} \in \obsSet{i}} 
                \obsFSub{i}(\st', \langle \actSub{i}, \actSub{-i} \rangle, \obsSub{i}) 
                V^{\polSub{i}}(\hist_{i, t}\actSub{i}\obsSub{i}) 
            \right] \\
            % put agent -i actions inside []
            &= \sum_{\st \in \stSpace} \sum_{\polSub{-i} \in \Pi_{-i}} \sum_{\hist_{-i, t} \in \histSpace_{i, t}}
                % \sum_{\st, \polSub{-i}, \hist_{-i, t}} 
                \belSub{i}(\langle \st, \polSub{-i}, \histSub{-i, t} \rangle | \hist_{i, t}) 
                \sum_{\actSub{i} \in \actSet{i}} \polSub{i}(\actSub{i} | \hist_{i, t})
                \left[
                    \sum_{\actSub{-i} \in \actSet{-i}} \polSub{-i}(\actSub{-i} | \hist_{-i, t})
                \rewFSub{i}(\st, \langle \actSub{i}, \actSub{-i} \rangle)
                \right. \\
            & \quad \left. 
                + \disc 
                \sum_{\actSub{-i} \in \actSet{-i}} \polSub{-i}(\actSub{-i} | \hist_{-i, t})
                \sum_{\st' \in \stSpace} \transF(\st, \langle \actSub{i}, \actSub{-i} \rangle, \st') 
                \sum_{\obsSub{i} \in \obsSet{i}}
                \obsFSub{i}(\st', \langle \actSub{i}, \actSub{-i} \rangle, \obsSub{i}) 
                V^{\polSub{i}}(\hist_{i, t}\actSub{i}\obsSub{i}) 
            \right] \\
            % add sum over agent -i obs
            % this doesn't change the equality since the probability sums to 1
            &= \sum_{\st \in \stSpace} \sum_{\polSub{-i} \in \Pi_{-i}} \sum_{\hist_{-i, t} \in \histSpace_{i, t}}
                % \sum_{\st, \polSub{-i}, \hist_{-i, t}} 
                \belSub{i}(\langle \st, \polSub{-i}, \histSub{-i, t} \rangle | \hist_{i, t}) 
                \sum_{\actSub{i} \in \actSet{i}} \polSub{i}(\actSub{i} | \hist_{i, t})
                \\
            & \quad \times \left[ 
                \sum_{\actSub{-i} \in \actSet{-i}} \polSub{-i}(\actSub{-i} | \hist_{-i, t})
                \rewFSub{i}(\st, \langle \actSub{i}, \actSub{-i} \rangle)
                + \disc 
                \sum_{\actSub{-i} \in \actSet{-i}} \polSub{-i}(\actSub{-i} | \hist_{-i, t})
                \sum_{\st' \in \stSpace} \transF(\st, \langle \actSub{i}, \actSub{-i} \rangle, \st') \right. \\
            & \quad \left. \times
                \sum_{\obsSub{-i} \in \obsSet{-i}} \obsFSub{-i}(\st', \langle \actSub{i}, \actSub{-i} \rangle, \obsSub{-i})
                \sum_{\obsSub{i} \in \obsSet{i}} \obsFSub{i}(\st', \langle \actSub{i}, \actSub{-i} \rangle, \obsSub{i}) 
                V^{\polSub{i}}(\hist_{i, t}\actSub{i}\obsSub{i}) 
            \right] \\
            % substitute in derived POMDP functions
            &= \sum_{\st \in \stSpace} \sum_{\polSub{-i} \in \Pi_{-i}} \sum_{\hist_{-i, t} \in \histSpace_{i, t}}
                % \sum_{\st, \polSub{-i}, \hist_{-i, t}} 
                \belSub{i}(\langle \st, \polSub{-i}, \histSub{-i, t} \rangle | \hist_{i, t}) 
                \sum_{\actSub{i} \in \actSet{i}} \polSub{i}(\actSub{i} | \hist_{i, t})
                \\
            & \quad \times \left[ 
                \bar{\rewF}(\langle \st, \polSub{-i}, \histSub{-i, t} \rangle, \act_{i})
                + \disc 
                \sum_{\actSub{-i} \in \actSet{-i}} 
                \sum_{\st' \in \stSpace} 
                \sum_{\obsSub{-i} \in \obsSet{-i}}
                % \sum_{\actSub{-i}, \st', \obsSub{i}, \obsSub{-i}} 
                \bar{\transF}(
                    \langle \st, \polSub{-i}, \histSub{-i, t} \rangle, 
                    \act_{i}, 
                    \langle \st', \polSub{-i}, \histSub{-i, t}\actSub{-i}\obsSub{-i} \rangle)
                \right. \\
            & \quad \left.
                \times 
                \sum_{\obsSub{i} \in \obsSet{i}} 
                \bar{\obsF}(\langle \st', \polSub{-i}, \histSub{-i, t}\actSub{-i}\obsSub{-i}\rangle, \act_{i}, \obs_{i})
                \bar{V}^{\polSub{i}}(\hist_{i, t}\actSub{i}\obsSub{i}) 
            \right] \\
            % substitute in history policy state notation
            &= \sum_{\histSt_{t}} 
                \belSub{i}(\histSt_{t} | \hist_{i, t})
                \sum_{\actSub{i} \in \actSet{i}} \polSub{i}(\actSub{i} | \hist_{i, t}) 
                \\
            & \quad \times \left[
                \bar{\rewF}(\histSt_{t}, \actSub{i}) 
                + \disc 
                \sum_{\histSt_{t+1}}
                \bar{\transF}(\histSt_{t}, \actSub{i}, \histSt_{t+1})
                \sum_{\obsSub{i} \in \obsSet{i}}
                \bar{\obsF}(\histSt_{t+1}, \act_{i}, \histSt_{t+1}.\obs_{i})
                \bar{V}^{\polSub{i}}(\hist_{i, t}\actSub{i}\obsSub{i}) 
            \right] \\
            &= \bar{V}^{\polSub{i}}(\hist_{i, t})
    \end{align*}    

\end{proof}

\subsection{Convergence}

Here we show show that the proposed algorithm \algacronym converges to the $\epsilon$-optimal solution of the POSG under the type-based reasoning assumptions. The main steps of our proof are adapted from \citep{silverMonteCarloPlanningLarge2010}, however we extend the original proof to the type-based, multi-agent setting. Note, our proof does not hold for POSGs in general, but only for POSGs under the type-based reasoning assumptions. When mentioning the POSG, we are referring to the POSG problem along with a set of stationary history-based policies $\Pi_{-i}$ for the other agent $-i$, and a prior distribution over this set $\rho$.

Our proof is based on showing that, for an arbitrary rollout policy $\polSub{i}$ for the planning agent $i$, the \textit{POSG rollout distribution} (the distribution over full histories of agent $i$ when performing root sampling of the state $\st$, and other agent policy $\polSub{-i}$ and history $\histSub{-i}$) is equal to the \textit{derived POMDP rollout distribution} (the distribution over full histories when sampling in the \hpsname POMDP). Given that these two distributions are identical, the statistics maintained in the search tree will converge to the same number in expectation. This allows us to apply the analysis for MCTS in POMDPs from \citet{silverMonteCarloPlanningLarge2010} to \algacronym in typed POSGs. 

\begin{definition}
    The \textbf{POSG rollout distribution}, $\mathcal{D}^{\polSub{i}}_{K}(\histSub{i, d})$, is the distribution over histories for planning agent $i$ in a POSG, where the other agent $-i$'s policy is from the set $\Pi_{-i}$ with a prior distribution over this set $\rho$, when performing Monte-Carlo simulations according to a policy $\polSub{i}$ in combination with root sampling an initial state, other agent policy, and other agent history. This distribution, for a particular time $d$, is given by $\mathcal{D}^{\polSub{i}}_{K}(\histSub{i, d}) := \frac{1}{K_{d}} \sum_{k=1}^{K} \mathbb{I}_{\histSub{i, d}}(\histSub{i, d}^{(k)})$, where $K$ is the number of simulations that comprise the empirical distribution, $K_{d}$ is the number of simulations that reach depth $d$ (not all simulations might be equally long), and $\histSub{i, d}^{(k)}$ is the history specified by the $k$-th particle at time $d$.
\end{definition}

\begin{definition}
    The \textbf{derived POMDP rollout distribution}, $\hat{\mathcal{D}}^{\polSub{i}}(\histSub{i, d})$, is the distribution over histories for planning agent $i$ in the derived POMDP, when performing Monte-Carlo simulations according to policy $\polSub{i}$ and sampling state transitions, observations, and rewards from $\hat{\mathcal{M}}$. This distribution, for a particular time $d+1$ history, is given by $\hat{\mathcal{D}}^{\polSub{i}}(\histSub{i, d}\actSub{i}\obsSub{i}) = \hat{\mathcal{D}}^{\polSub{i}}(\histSub{i, d}) \polSub{i}(\actSub{i}|\histSub{i, d}) \sum_{\histSt_{d} \in \histStSpace_{d}} \belSub{i}(\histSt_{d} | \histSub{i, d}) \sum_{\histSt_{d+1} \in \histStSpace_{d+1}} \bar{\transF}(\histSt_{d}, \act_{i}, \histSt_{d+1}) \bar{\obsF}(\histSt_{d+1} \act_{i}, \obs_{i})$.
\end{definition}

Now, our main theoretical result is that these distributions are the same in the limit of the number of simulations.

\begin{lemma}
    Given a POSG $\mathcal{M}$, a set of stationary history-based policies $\Pi_{-i}$ for the other agent $-i$, and a prior distribution over this set $\rho$, for any rollout policy for planning agent $i$, $\polSub{i}$, the POSG rollout distribution converges in probability to to the derived POMDP rollout distribution, $\forall \polSub{i}, \histSub{i, d}\ \mathcal{D}^{\polSub{i}}_{K}(\histSub{i, d}) \xrightarrow{p} \hat{\mathcal{D}}^{\polSub{i}}(\histSub{i, d})$.

    \label{lemma:rollout_dists}
\end{lemma}

\begin{proof}
    By forward induction from $\histSub{i, t}$, where $t$ is the timestep at the root.
    \smallskip
    
    \textit{Base case:} At the root when ($d=t$, $\histSub{i,d} = \histSub{i,t}$), it is clear that $\mathcal{D}^{\polSub{i}}_{K}(\histSub{i, d}) = \hat{\mathcal{D}}^{\polSub{i}}(\histSub{i, d}) = 1$ since all simulations go through the root node.
    \smallskip
    
    \textit{Step case:} Assume $\mathcal{D}^{\polSub{i}}_{k}(\histSub{i, d}) = \hat{\mathcal{D}}^{\polSub{i}}(\histSub{i, d})$ for all time $d$ histories where $t \leq d < T$. Consider any time $d+1$ history $\histSub{i, d+1} = \histSub{i, d}\actSub{i}\obsSub{i}$, the following relation holds:
    
    \begin{equation}
    \begin{aligned}
        \mathcal{D}^{\polSub{i}}_{K}(\histSub{i, d}\actSub{i}\obsSub{i}) 
            &= \mathcal{D}^{\polSub{i}}_{K}(\histSub{i, d})
                \polSub{i}(\actSub{i}|\histSub{i, d}) 
                \sum_{\st \in \stSpace} 
                \sum_{\polSub{-i} \in \Pi_{-i}}     
                \sum_{\hist_{-i, d} \in \histSpace_{i, d}} 
                \belSub{i}(\langle \st, \polSub{-i}, \histSub{-i, d}\rangle | \histSub{i, d}) 
                \sum_{\actSub{-i} \in \actSet{-i}} \polSub{-i}(\actSub{-i} | \hist_{-i, d}) \\
            & \quad \times
                \sum_{\st' \in \stSpace}    
                \transF(\st, \langle \actSub{i}, \actSub{-i} \rangle, \st') 
                \obsFSub{i}(\st', \langle \actSub{i}, \actSub{-i} \rangle, \obsSub{i}) \\
            % Add sum over other agent -i obs
            &= \mathcal{D}^{\polSub{i}}_{K}(\histSub{i, d})
                \polSub{i}(\actSub{i}|\histSub{i, d}) 
                \sum_{\st \in \stSpace} 
                \sum_{\polSub{-i} \in \Pi_{-i}}     
                \sum_{\hist_{-i, d} \in \histSpace_{i, d}} 
                \belSub{i}(\langle \st, \polSub{-i}, \histSub{-i, d}\rangle | \histSub{i, d}) 
                \sum_{\actSub{-i} \in \actSet{-i}} \polSub{-i}(\actSub{-i} | \hist_{-i, d}) \\
            & \quad \times
                \sum_{\st' \in \stSpace}
                \transF(\st, \langle \actSub{i}, \actSub{-i} \rangle, \st') 
                \obsFSub{i}(\st', \langle \actSub{i}, \actSub{-i} \rangle, \obsSub{i}) 
                \sum_{\obsSub{-i} \in \obsSet{-i}}
                \obsFSub{-i}(\st', \langle \actSub{i}, \actSub{-i} \rangle, \obsSub{-i}) \\
            % Substitute in POMDP functions
            &= \hat{\mathcal{D}}^{\polSub{i}}(\histSub{i, d})
                \polSub{i}(\actSub{i}|\histSub{i, d}) 
                \sum_{\st \in \stSpace} 
                \sum_{\polSub{-i} \in \Pi_{-i}}     
                \sum_{\hist_{-i, d} \in \histSpace_{i, d}} 
                \belSub{i}(\langle \st, \polSub{-i}, \histSub{-i, d}\rangle | \histSub{i, d}) \\
            & \quad \times
                \sum_{\st' \in \stSpace}
                \sum_{\actSub{-i} \in \actSet{-i}} 
                \sum_{\obsSub{-i} \in \obsSet{-i}}
                \bar{\transF}(
                    \langle \st, \polSub{-i}, \histSub{-i, d} \rangle, 
                    \act_{i}, 
                    \langle \st', \polSub{-i}, \histSub{-i, d}\actSub{-i}\obsSub{-i} \rangle
                )
                \bar{\obsF}(
                    \langle \st', \polSub{-i}, \histSub{-i, d}\actSub{-i}\obsSub{-i} \rangle, 
                    \act_{i}, 
                    \obs_{i}
                ) \\
            % Substitute in hps functions
            &= \hat{\mathcal{D}}^{\polSub{i}}(\histSub{i, d})
                \polSub{i}(\actSub{i}|\histSub{i, d}) 
                \sum_{\histSt_{d} \in \histStSpace_{d}} 
                \belSub{i}(\histSt_{d} | \histSub{i, d})
                \sum_{\histSt_{d+1} \in \histStSpace}
                \bar{\transF}(\histSt_{d}, \act_{i}, \histSt_{d+1})
                \bar{\obsF}(\histSt_{d+1} \act_{i}, \obs_{i}) \\
            &= \hat{\mathcal{D}}^{\polSub{i}}(\histSub{i, d}\actSub{i}\obsSub{i})
    \end{aligned}
    \end{equation}

    Where the third line is obtained using the induction hypothesis, and the rest from the definitions.

\end{proof}

Now that we have shown that the rollout distributions are equivalent between the POSG and derived POMDP, we can present the main convergence result.

\begin{theorem}
    For all $\epsilon > 0$ (the numerical precision, see \aref{alg:search}), given a suitably chosen c (e.g. $c > \frac{R_{max}}{1-\gamma}$) and prior probabilities $P(\actSub{i}|\histSub{i}) > 0, \forall \histSub{i} \in \histSpace_{i}, \actSub{i} \in \actSet{i}$ (e.g. $\lambda > 0$),
    from history $\histSub{i}$ \algacronym constructs a value function at the root node that converges in probability to an $\epsilon'$-optimal value function, $V(\histSub{i}) \xrightarrow{p} V^{*}_{\epsilon'}(\histSub{i})$, where $\epsilon' = \frac{\epsilon}{1 - \gamma}$. As the number of visits $N(\histSub{i})$ approaches infinity, the bias of $V(\histSub{i})$ is $O(\log{N(\histSub{i})}/N(\histSub{i})).$
\end{theorem}

\begin{proof}
    By \ref{lemma:rollout_dists} the \algacronym simulations can be mapped to the PUCT simulations in the derived POMDP. By \ref{lemma:posg_bapomdp} a POSG with a set of stationary policies for the other agent, and a prior over this set is a POMDP, so the analysis from \citet{silverMonteCarloPlanningLarge2010} applies to \algacronym, noting that for $N(\histSub{i})$ sufficiently large PUCB has the same regret bounds as UCB given each action is given prior probability $P(\histSub{i}\actSub{i}) > 0, \forall \histSub{i} \in \histSpace_{i}, \actSub{i} \in \actSet{i}$ (\citet{rosinMultiarmedBanditsEpisode2011}, Thm. 2 and Cor. 2) and so the same analysis of POMCP using UCT applies to POMCP using PUCT. 
\end{proof}

\subsection{Belief Update}

In this section we provide the equations for the initial belief and belief updates using \hpsnames. This is not part of the proof of convergence, but is provided for reference. Conventions differ between AI communities and problem domain regarding whether each episode begins with the agent performing an action $\actSub{i, 0}$ or receiving an initial observation $\obsSub{i, 0}$. So here we show the initial belief for both conventions. Note that in the observation-first setting the initial belief requires an initial observation function $\obsFSub{i, 0}(\obsSub{i}, \st)$ defining the probability agent $i$ received initial observation $\obsSub{i}$ given initial state $\st$. It is always possible to convert between the two conventions. An observation-first model can be converted into an action-first model by including a unique initial state (or state feature) such that all actions in that state have the effect of transitioning to a true initial state and the agent receiving an initial observation. Similarly, an action-first model can be converted to an observation-first model by all agents receiving some unique initial observation (or random initial observation independent of the initial state) at the initial timestep. 

\begin{proposition}(\textbf{Initial Belief})
    Given a POSG $\mathcal{M} = \posgTuple$, where the other agent $-i$ is using a policy from a known set of policies $\Pi_{-i} = \{\polSub{-i, m} | m = 1, \cdots, M\}$, and a prior distribution over them $\rho$, where $\rho(\polSub{-i,m}) = Pr(\polSub{-i, m})$.
    
    (\textbf{Action-first}) If the POSG begins each episode with each agent performing an action, the initial belief for agent $i$ over \hpsnames with initial history $\histSub{i, 0} = \emptyset$ is,
    
    \begin{align}
    \belSub{i, 0}(\histSt | \histSub{i, 0}) &= 
        \begin{cases}
            \belInit(\histSt.\st) \rho(\histSt.\polSub{-i}) & \text{ if } \histSt.\histSub{-i} = \emptyset \\
            0 & \text{ otherwise}. \\
        \end{cases} &&
    \end{align}
    \smallskip
    
    (\textbf{Observation-first)} If the POSG begins each episode with each agent receiving an observation, the initial belief for agent $i$ over \hpsnames with initial history $\histSub{i, 0} = \obsSub{i, 0}$ is,
    
    \begin{align}
    \belSub{i, 0}(\histSt | \histSub{i, 0}) &= 
        \begin{cases}
            \eta \belInit(\histSt.\st) \rho(\histSt.\polSub{-i}) \obsFSub{-i, 0}(\histSt.\obsSub{-i}, \histSt.\st) & \text{ if } \histSt.\histSub{-i} = \obsSub{-i} \\
            0 & \text{ otherwise} \\
        \end{cases} &&
    \end{align}
    \smallskip
    
    Where $\obsFSub{-i, 0}(\obsSub{-i}, \st) = Pr(\obsSub{-i} | \st)$ is the initial observation function for agent $-i$, and $\eta = 1 / Pr(\obsSub{i, 0} | \belInit)$ is a normalizing constant with $Pr(\obsSub{i, 0} | \belSub{0}) = \sum_{\st \in \stSpace} \belInit(\st) \obsFSub{i, 0}(\obsSub{i, 0}, \st)$.
    
    \label{prop:initial_belief}
\end{proposition}
\bigskip

\begin{proposition}(\textbf{Belief Update})
    Given a POSG $\mathcal{M} = \posgTuple$, where the other agent $-i$ is using a policy from a known set of policies $\Pi_{-i} = \{\polSub{-i, m} | m = 1, \cdots, M\}$, and a prior distribution over them $\rho$, where $\rho(\polSub{-i,m}) = Pr(\polSub{-i, m})$. The belief for agent $i$ with history $\histSub{i, t} = \histSub{i, t-1}\actSub{i,t-1}\obsSub{i,t}$ after $t$ time-steps is,
    
    \begin{equation}
    \begin{aligned}
        \belSub{i, t}(\histSt_{t} | \histSub{i, t}) &=     
            \beta \histSt_{t}.\polSub{-i}(\histSt_{t}.\actSub{-i, t-1} | \histSt_{t}.\histSub{-i, t-1}) \\
            & \quad \times \obsFSub{i}(\obsSub{i, t}, \langle \actSub{i, t-1}, \histSt_{t}.\actSub{-i, t-1} \rangle, \histSt_{t}.\st) 
            \obsFSub{-i}(\histSt_{t}.\obsSub{-i, t}, \langle \actSub{i, t-1}, \histSt_{t}.\actSub{-i, t-1} \rangle, \histSt_{t}.\st) \\
            & \quad \times \sum_{\st \in \stSpace} \transF(\histSt_{t}.\st, \langle \actSub{i, t-1}, \histSt_{t}.\actSub{-i, t-1} \rangle, \st) 
            \belSub{i, t-1}(\langle \st, \histSt_{t}.\polSub{-i}, \histSt_{t}.\histSub{-i, t-1} \rangle \rangle | \histSub{i, t-1})
    \end{aligned}
    \end{equation}
    \smallskip
    
    Where $\beta$ is a normalizing constant.
    
    \label{prop:belief_update}
\end{proposition}

\section{Environments} \label{sup:envs}

In this section we describe the environment used for experiments in more detail.

\begin{figure*}[ht]
    \centering
    \begin{subfigure}{.25\linewidth}
        \centering
        \includegraphics[width=\linewidth, height=\envfigheight, keepaspectratio]{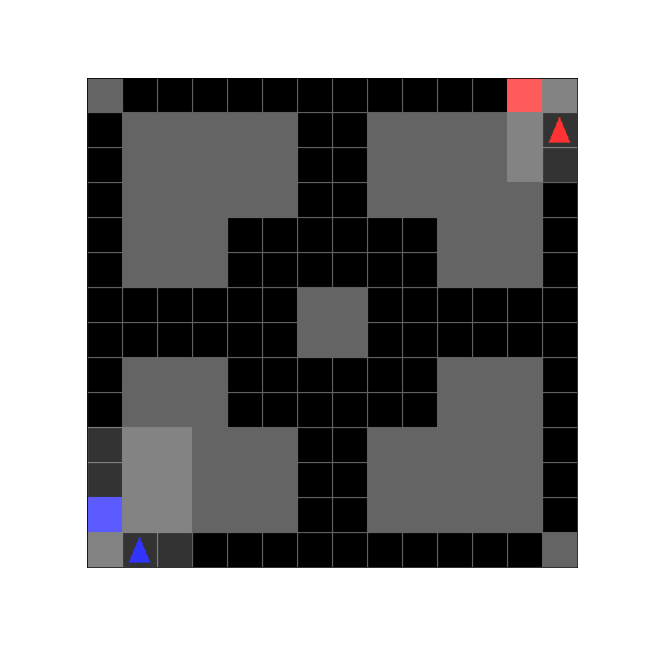}
        \caption{Driving}
    \end{subfigure}%
    \begin{subfigure}{.25\linewidth}
        \centering
        \includegraphics[width=\linewidth, height=\envfigheight, keepaspectratio]{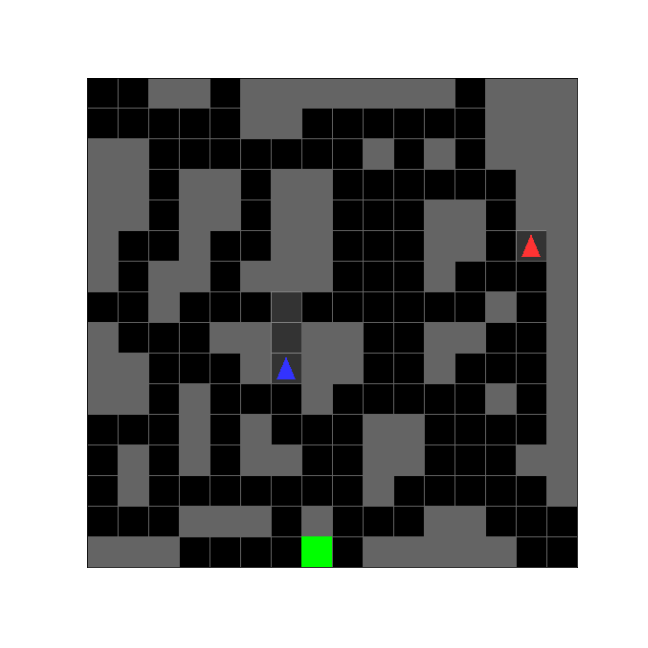}
        \caption{Pursuit-Evasion (PE)}
    \end{subfigure}%
    \begin{subfigure}{.25\linewidth}
        \centering
        \includegraphics[width=\linewidth, height=\envfigheight, keepaspectratio]{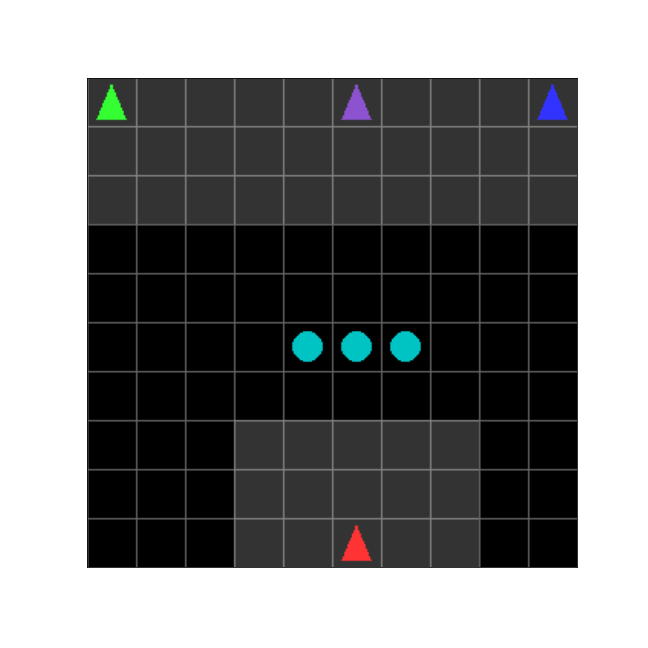}
        \caption{Predator-Prey (PP)}
    \end{subfigure}
    \caption{Experiment Environments}
    \label{supfig:envs}
\end{figure*}

\textbf{Driving:} A general-sum 2D grid world navigation problem requiring coordination \citep{lererLearningExistingSocial2019}. Each agent controls a vehicle and is tasked with driving the vehicle from start to destination locations while avoiding crashing into other vehicles. Each agent observes their local area, speed, and destination, and whether they've reached their destination or crashed. Agents receive a reward of $1$ if they reach their destination, $-1$ if they crash into another vehicle, and $-0.05$ if they attempt to move into a wall. To reduce exploration difficulty agents receive a reward of $0.05$ each time they make progress towards their destination. The exploration bonus is analogous to GPS in that it provides guidance for navigation but not for how to coordinate with other vehicles. Episodes ended when all agents had either crashed into another vehicle or reached their destination, or 50 steps had passed.  

\textbf{Pursuit-Evasion (PE):} An asymmetric zero-sum grid world problem involving two agents, an evader and a pursuer \citep{seamanNestedReasoningAutonomous2018, schwartz2022intmcp}. The evader's goal is to reach a safe location, while the pursuer's aim is to spot the evader before it reaches its goal. The evader is considered caught if it is observed by the pursuer. Both agents have knowledge of each others starting locations, however, only the evader has knowledge of its goal location. The pursuer only knows the set of possible safe locations. Thus, this environment requires each agent to reason about the which path the other agent will take through the dense grid environment. Each agent receives six bits of observation per step. Four bits indicate whether there is a wall or not in each of the cardinal directions, one bit indicates whether the opponent can be seen in the agent's field of vision (which is a cone in front of the agent), and the final bit indicates whether the opponent can be heard within Manhattan distance two of the agent. Due to the lack of precision of these observations, the pursuer never knows the exact position of the evader and vice versa. Similar to the Driving environment the evader agent receives a small bonus whenever it makes progress towards the safe location, while the pursuer receives the opposite reward. Episodes ended when the evader was captured or reached the safe location, or 100 steps had passed.

\textbf{Predator-Prey (PP):} A co-operative grid world problem involving multiple predator agents working together to catch prey \citep{loweMultiagentActorcriticMixed2017}. Prey are controlled autonomously and preference movement away from any observable predators or other prey. Predators can catch prey by being in an adjacent cell, with the number of predators required to catch a prey based on the prey strength. Both predators and prey can observe a 5-by-5 area around themselves, namely whether each cell contains a wall, predator, prey, or is empty. Each prey capture gives all predators a reward of $1/N_{prey}$. Predators start each episode from random separate locations along the edge of the grid, while prey start together in the center of the grid. We ran experiments on two different versions of the environment, where both versions had three prey. The \textit{two-agent} version had two predators with each prey requiring two predators to capture. The \textit{four-agent} version had four predators with prey requiring three predators to capture. Both versions required coordination between agents to capture the prey. Episodes ended when all prey were captured, or 50 steps had passed.

\begin{table}[ht]
    \centering
    \caption{Environment State, Action, Observation Space Sizes. State and Observation Sizes for the Driving and PP Environments are Approximate, but Correct to Within an Order of Magnitude.}
    \begin{tabular}{|c|c|c|c|}
        \hline
        Environment & $|\stSpace|$ & $|\actSet{i}|$ & $|\obsSet{i}|$ \\
        \hline
        Driving & $2.7\times10^{8}$ & $5$ & $3\times10^{6}$ \\
        PE & $2.3\times10^{7}$ & $4$ & $64$ \\
        PP (two-agents) & $7\times10^{10}$ & $5$ & $7\times10^{6}$ \\
        PP (four-agents) & $6\times10^{14}$ & $5$ & $6\times10^{10}$ \\
        \hline
    \end{tabular}
    \label{suptab:env_sizes}
\end{table}

\section{Policies} \label{sup:fixed_policies}

For each environment, we created four to five policies to be used for the set $\Pi$. This set was used for the other agent policies during evaluations and also for the meta-policy $\sigma_{i}$ and policy prior $\rho$. Each policy was represented using a deep neural network with an actor-critic architecture (with policy and value function outputs). In the following sections we supply additional details for each set of policies including multi-agent training schemes, training hyperparameters, and empirical-game payoff matrices. 

\subsection{Training}

Each policy was trained using multi-agent reinforcement learning. We used different multi-agent training schemes for each environment, while the same RL algorithm was used for training each individual policy. Specifically, for each individual policy we used the Rllib \citep{liangRLlibAbstractionsDistributed2018} implementation of the Proximal Policy Optimization (PPO) model-free, policy-gradient method \citep{schulmanProximalPolicyOptimization2017}. We used the same neural network architecture for all policies, namely two fully-connected layers with 64 and 32 units, respectively, followed by a 256 unit LSTM, whose output was fed into two separate fully connected output heads, one for the policy and one for the value function. The neural network architecture and training hyperparameters are shown in \tref{suptab:policy_training_hyperparameters}. All policies were trained until convergence, as indicated by their learning curves.

\begin{table}[H]
    \centering
    \caption{Policy Training Hyperparameters.}
    \begin{tabular}{|c|c|c|c|}
        \hline
        Hyper parameter & Driving & PE & PP \\
        \hline
        Training steps & \multicolumn{3}{|c|}{10,240,000} \\
        Fully Connected Network Layers & \multicolumn{3}{|c|}{[64, 32]} \\
        LSTM Cell Size & \multicolumn{3}{|c|}{256} \\
        Learning Rate & \multicolumn{3}{|c|}{0.0003} \\
        KL Coefficient & \multicolumn{3}{|c|}{0.2} \\
        KL target & \multicolumn{3}{|c|}{0.01} \\
        Batch size & \multicolumn{3}{|c|}{2048} \\
        LSTM training sequence length & \multicolumn{3}{|c|}{20} \\
        Entropy Bonus Cofficient & \multicolumn{3}{|c|}{0.001} \\
        Clip param & \multicolumn{3}{|c|}{0.3} \\
        $\disc$ & 0.99 & 0.99 & $[0.99, 0.999]$ \\
        GAE $\lambda$ & 0.9 & 0.9 & $[0.90, 0.95]$ \\
        SGD Minibatch size & 256 & 256 & $[256, 512]$ \\
        Num. SGD Iterations & 10 & 10 & $[10, 2]$ \\
        \hline
    \end{tabular}
    \label{suptab:policy_training_hyperparameters}
\end{table}

\begin{figure}[H]
    \centering
    \begin{subfigure}[t]{0.23\textwidth}
        \centering
        \includegraphics[width=\textwidth, height=0.2\textheight, keepaspectratio]{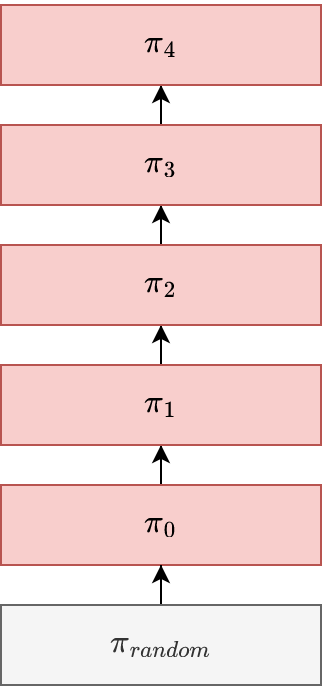}
        \caption{Driving}
    \end{subfigure}
    \begin{subfigure}[t]{0.48\textwidth}
        \centering
        \includegraphics[width=\textwidth, height=0.2\textheight, keepaspectratio]{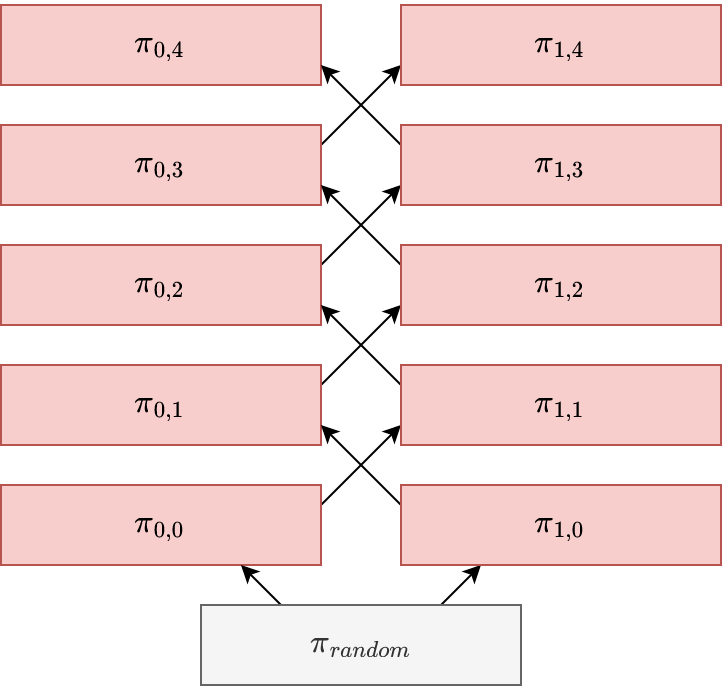}
        \caption{Pursuit-Evasion}
    \end{subfigure}
    \begin{subfigure}[t]{0.23\textwidth}
        \centering
        \includegraphics[width=\textwidth, height=0.2\textheight, keepaspectratio]{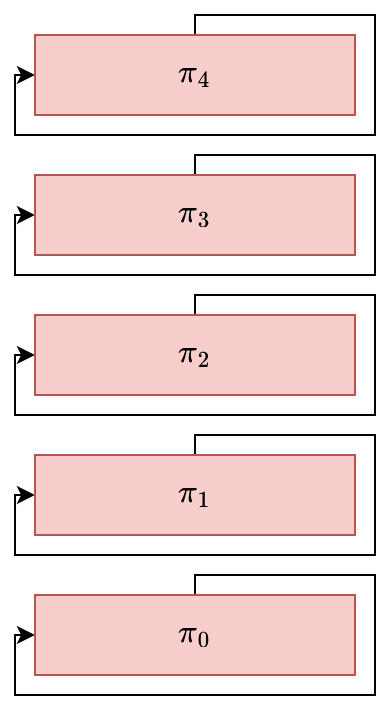}
        \caption{Predator-Prey}
    \end{subfigure}
    \caption{Multi-agent training schemas used for generating the fixed policies for the different environments (adapted from \citep{cuiKlevelReasoningZeroShot2021}). For the Pursuit-Evasion environment separate policies were trained for the \textit{Evader} (agent 0) and the \textit{Pursuer} (agent 1) agents.}
    \label{supfig:training_schemas}
\end{figure}

\begin{figure}[H]
    \centering
    \begin{subfigure}[t]{0.195\textwidth}
        \centering
        \includegraphics[width=\textwidth, height=0.6\textheight, keepaspectratio]{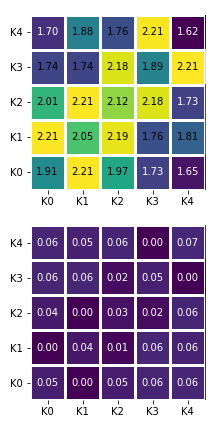}
        \caption{Driving}
    \end{subfigure}
   \begin{subfigure}[t]{0.195\textwidth}
        \centering
        \includegraphics[width=\textwidth, height=0.6\textheight,keepaspectratio]{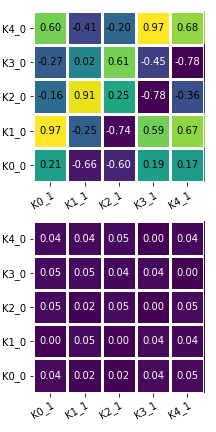}
        \caption{PE (Evader)}
    \end{subfigure}
    \begin{subfigure}[t]{0.195\textwidth}
        \centering
        \includegraphics[width=\textwidth, height=0.6\textheight,keepaspectratio]{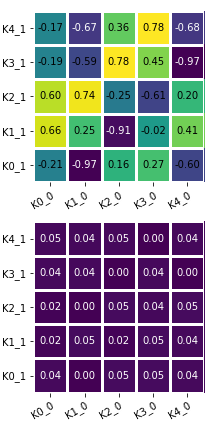}
        \caption{PE (Pursuer)}
    \end{subfigure}
    \begin{subfigure}[t]{0.195\textwidth}
        \centering
        \includegraphics[width=\textwidth, height=0.6\textheight, keepaspectratio]{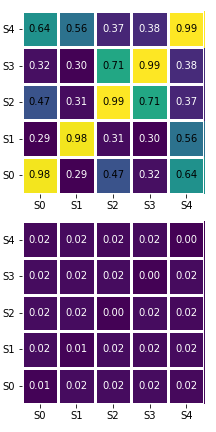}
        \caption{PP (2 agents)}
    \end{subfigure}
    \begin{subfigure}[t]{0.195\textwidth}
        \centering
        \includegraphics[width=\textwidth, height=0.6\textheight, keepaspectratio]{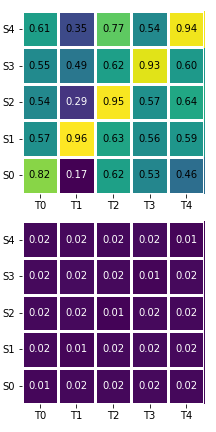}
        \caption{PP (4 agents)}
    \end{subfigure}
    \caption{Empirical-game payoff matrices for the set of policies for each environment. The top row shows the mean payoff after 1000 evaluation episodes, while the bottom row shows the corresponding $95\%$ confidence intervals. Each cell shows the result for the row policy when paired with the column policy or team of policies.}
    \label{supfig:pi_payoffs}
\end{figure}

\subsection{Driving Policies}

For the Driving experiments we trained a set of five $K$-level reasoning (KLR) policies. In the KLR training scheme, policies are trained in a hierarchy, the level $K=0$ policy is trained against a uniform random policy, level $K=1$ is trained against the level $K=0$, and so on with the level $K$ policy trained as a best response to the level $K-1$ policy for $K > 0$. We trained policies synchronously using the Synchronous KLR training method \citep{cuiKlevelReasoningZeroShot2021}. \fref{supfig:training_schemas} provides a visualization of the training schema used. \fref{supfig:pi_payoffs} (a) shows the pairwise performance for the Driving environment policies, with each policy evaluated against each other policy for 1000 episodes.

For the policy prior $\rho$, we used a uniform distribution over the policies with reasoning levels $k = 0, 1, 2, 3$. The meta-policy was defined using all five policies, which included the level $k=4$ policy. This meant the planning agent had access to a best-response for all the policies in $\rho$ and allowed a fair comparison against the Best-Response baseline.

\subsection{Pursuit Evasion Policies}

For the PE experiments we trained a set of five KLR policies, similar to the Driving experiments. The only difference being that we trained separate policies for the Evader and Pursuer at each reasoning level. \fref{supfig:training_schemas} provides a visualization of the training schema used. Separate policies were used because the PE problem is asymmetric, with the pursuer and evader having different objectives. \fref{supfig:pi_payoffs} shows the pairwise performance for the PE environment policies. The meta-policy and prior $\rho$ were defined the same as for the Driving environment.

\subsection{Predator Prey Policies}

For this fully cooperative problem we trained five independent teams of agents using self-play \citep{tesauroTDGammonSelfteachingBackgammon1994} where each team consisted of identical copies of the same policy. The policy architecture and hyperparameters were the same across each of the five teams, except for the initial random seed which was different. Using a different seed meant each team converged to a different solution and lead to a diverse set of policies (as shown by the empirical-game payoffs). We trained separate policies for the two-agent and four-agent versions. \fref{supfig:training_schemas} provides a visualization of the training schema used. \fref{supfig:pi_payoffs} shows the pairwise performance for the set of policies in each version of the environment. For the four-agent version we show the results from matching the row policy with a team of three versions of the same policy (e.g. T0 is three copies of policy S0).

We used a uniform distribution over the five teams for the prior $\rho$, with each team made up of copies of the same policy from set of five trained self-play policies - one copy in the two-agent version, and three copies in the four-agent version. This setup tested the planning agent's \textit{ad-hoc teamwork} ability. The meta-policy was defined using all five policies in both versions.

\section{Evaluation of exploration noise levels} \label{sup:lambda}

\fref{supfig:lambda_all} and \fref{supfig:lambda_std} show the performance of \algacronym using different values for the $\lambda$ hyperparameter.  

\begin{figure}[H]
    \centering
    \includegraphics[width=\linewidth, height=\textheight, keepaspectratio]{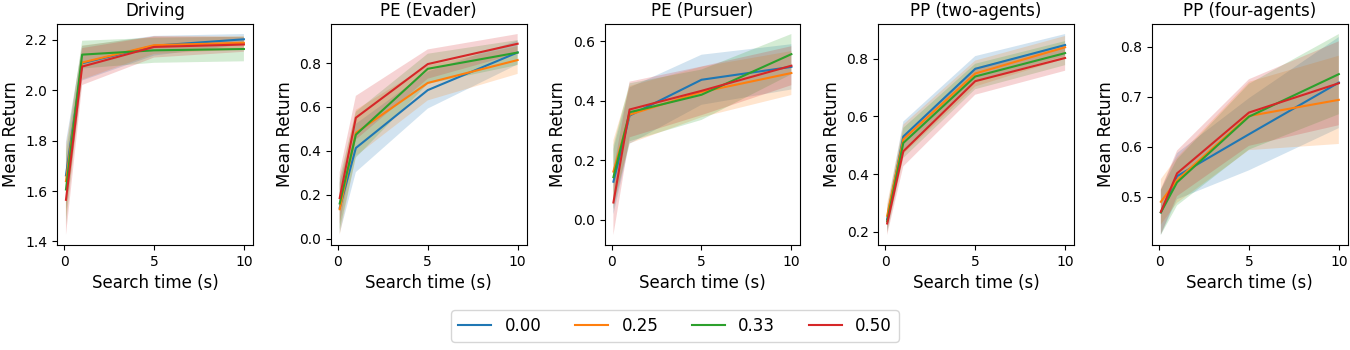}
    \caption{Comparison of \algacronym using different values for $\lambda$ in environment. Each figure shows performance of \algacronym using the $\sigma^{0.25}$ meta-policy. Shaded areas show the $95\%$ confidence interval.}
    \label{supfig:lambda_all}
\end{figure}

\begin{figure}[H]
    \centering
    \includegraphics[width=\linewidth, height=0.2\textheight, keepaspectratio]{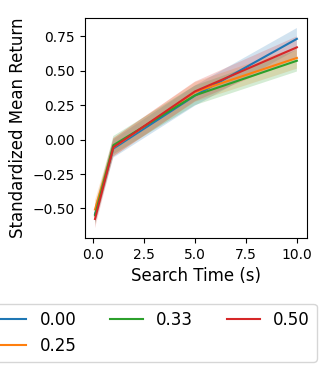}
    \caption{Standardized mean episode returns of \algacronym using different values for $\lambda$. Shaded areas show $95 \%$ CI.}
    \label{supfig:lambda_std}
\end{figure}

\section{Search depth} \label{sup:search_depth}

\fref{supfig:search_depth} shows the maximum search depth by planning time for \algacronym and baseline planning methods.

\begin{figure}[H]
    \centering
    \includegraphics[width=\linewidth, height=0.2\textheight, keepaspectratio]{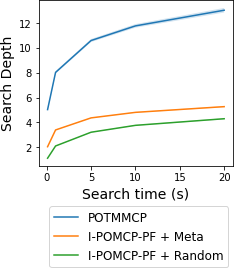}
    \caption{Maximum search depth vs planning time of \algacronym and I-POMCP-PF, averaged across all environments. Shaded areas show the $95\%$ confidence interval.}
    \label{supfig:search_depth}
\end{figure}

\section{Evaluation of different meta-policies} \label{sup:meta_pi}

\fref{supfig:meta_pi_comparison} shows the performance of \algacronym using the different meta-policies in each environment.  

\begin{figure}[H]
    \centering
    \includegraphics[width=\linewidth, height=\textheight, keepaspectratio]{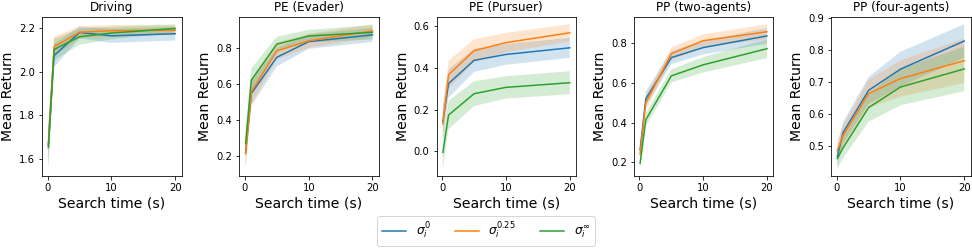}
    \caption{Performance of \algacronym with greedy $\sigma^{0}$, softmax $\sigma^{0.25}$, and uniform $\sigma^{\infty}$ meta-policies in each environment. Shaded areas show the $95\%$ CI.}
    \label{supfig:meta_pi_comparison}
\end{figure}

\section{Evaluation of different search policies} \label{sup:meta_vs_fixed}

\fref{supfig:meta_vs_all_fixed_pi} and \fref{supfig:meta_vs_best_worst_fixed_pi} compares the performance of \algacronym using different search policies.  

\begin{figure}[H]
    \centering
    \includegraphics[width=\linewidth, height=\textheight, keepaspectratio]{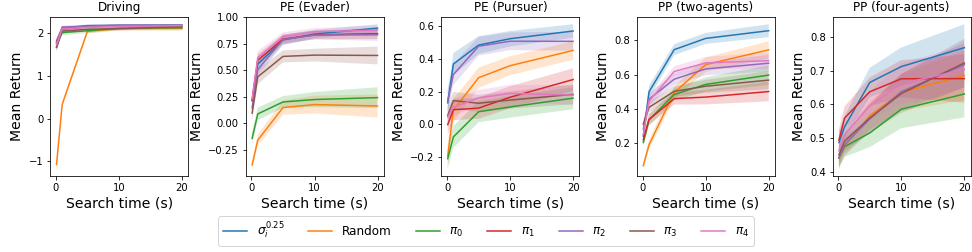}
    \caption{Comparison of \algacronym using different search policies in each environment. Each figure shows performance of \algacronym using the best performing meta-policy, the uniform random policy, and each of the available fixed policies. Shaded areas show the $95\%$ confidence interval.}
    \label{supfig:meta_vs_all_fixed_pi}
\end{figure}

\begin{figure}[H]
    \centering
    \includegraphics[width=\linewidth, height=\textheight, keepaspectratio]{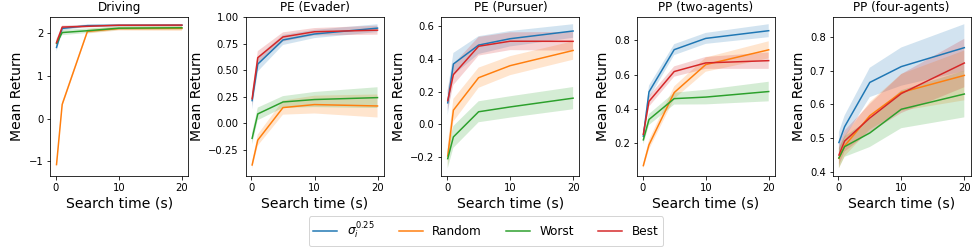}
    \caption{Comparison of \algacronym using different search policies in each environment. Each figure shows performance of \algacronym using the best performing meta-policy, the uniform random policy, and the best and worst of the available fixed policies (based on their performance given the maximum planning time). Shaded areas show the $95\%$ confidence interval.}
    \label{supfig:meta_vs_best_worst_fixed_pi}
\end{figure}

\section{Belief Accuracy} \label{sup:belief_accuracy}

\fref{supfig:bayes_accuracy} shows \algacronym's belief accuracy for each environment.

\begin{figure}[H]
    \centering
    \includegraphics[width=\linewidth, height=\textheight, keepaspectratio]{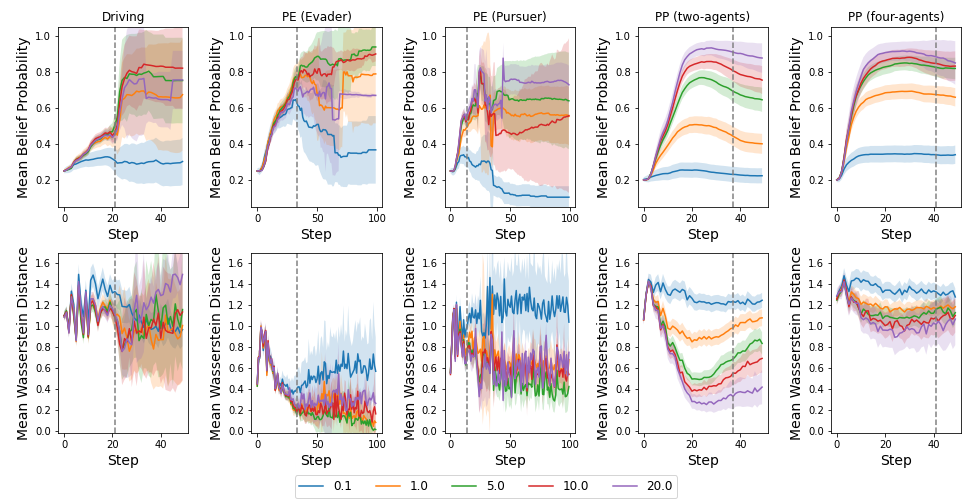}
    \caption{\algacronym's belief accuracy during an episode in each environment. (Top row) Mean probability assigned by \algacronym's belief to the true policy of the other agent. (Bottom row) Mean Wasserstein distance between the belief's estimated action distribution and the other agent's true action distribution. Each line in each figure is \algacronym using a different amount of search time. Shaded areas show $95\%$ confidence intervals. The vertical dashed line shows the mean number of steps taken per episode by \algacronym using 20 s of planning time. In all environments episode lengths were often shorter than the max episode lengths, leading to larger confidence intervals for steps later in the episode.}
    \label{supfig:bayes_accuracy}
\end{figure}

\section{Sensitivity to Novel Policies} \label{sup:sensitivity}

These results were not included in the main paper since they were outside the assumptions of our method. However, we include them here for reference and to motivate future research.

\fref{supfig:sensitivity} show the performance of \algacronym and baselines when paired with other agents using policies from $\hat{\Pi}_{-i}$ that are not included in the set of known policies $\Pi_{-i}$. The policies in $\hat{\Pi}_{-i}$ were generated in the same way as those within the set $\Pi_{-i}$, but with a different seed leading to differences in behaviours. From the results we can see that \algacronym is more robust to the out-of-distribution policies than the I-POMCP-PF baselines, and shown by the higher mean return. However, \algacronym's performance against the new policies $\hat{\Pi}_{-i}$ is significantly lower when compared to its performance against the known $\Pi_{-i}$ policies (\fref{fig:vs_baselines}), suggesting that \algacronym is sensitive to out-of-distribution policies at least for the policy prior used in our experiments. This presents an interesting follow-up question. Specifically, what types of policy priors can lead to more robust performance?

\begin{figure}[H]
    \centering
    \includegraphics[width=\linewidth, height=\resultfigheight, keepaspectratio]{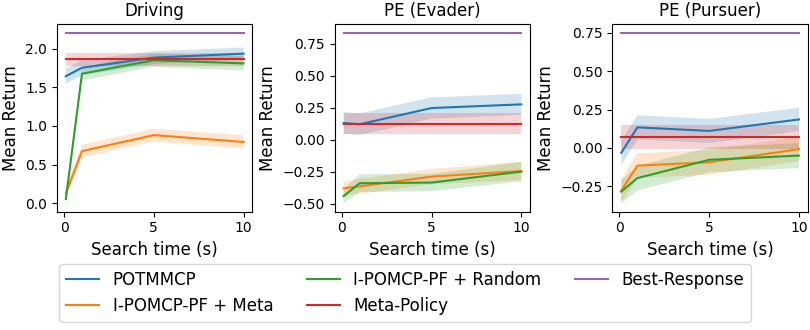}
    \caption{Mean episode return of \algacronym and baseline methods in the Driving and Pursuit-Evasion environments when paired with other agent policies outside of the known set $\Pi_{-i}$. Results are for \algacronym and baselines using the softmax $\sigma_{i}^{0.25}$ meta-policy. Shaded areas show the $95\%$ CI.}
    \label{supfig:sensitivity}
\end{figure}

\end{document}